\documentclass[twoside]{article}

\usepackage[accepted]{aistats2025}
\usepackage[round]{natbib}

\usepackage{booktabs}
\usepackage{makecell}
\usepackage{multirow}
\usepackage{color,xcolor}
\definecolor{mydarkblue}{rgb}{0,0.08,0.45}
\usepackage[colorlinks=true,
    linkcolor=mydarkblue,
    citecolor=mydarkblue,
    filecolor=mydarkblue,
    urlcolor=mydarkblue]{hyperref}
\usepackage[capitalise]{cleveref}
\crefname{assumption}{Assumption}{Assumptions}
\crefname{statement}{Statement}{Statements}
\usepackage{algorithm}
\usepackage[noend]{algorithmic}
\usepackage{comment}
\usepackage{amsmath,amsfonts,amssymb}
\usepackage{mathtools}
\usepackage{amsthm}
\usepackage{thmtools,thm-restate}

\newtheorem{proposition}{Proposition}

\newtheorem{assumption}{Assumption}
\newtheorem{lemma}{Lemma}

\newtheorem{example}{Example}

\newcommand{\cl}{\,:\,}
\usepackage{enumitem}

\newcommand\independent{\protect\mathpalette{\protect\independenT}{\perp}}
    \def\independenT#1#2{\mathrel{\rlap{$#1#2$}\mkern2mu{#1#2}}}
\newcommand\notindependent{\!\perp\!\!\!\!\not\perp\!}

\usepackage{minitoc}

\usepackage{tikz}
\usetikzlibrary{arrows,fit,positioning,shapes,bayesnet}
\usetikzlibrary{automata}
\newdimen\arrowsize
\pgfarrowsdeclare{arcsq}{arcsq}
{
	\arrowsize=0.2pt
	\advance\arrowsize by .5\pgflinewidth
	\pgfarrowsleftextend{-4\arrowsize-.5\pgflinewidth}
	\pgfarrowsrightextend{.5\pgflinewidth}
}
{
	\arrowsize=1.5pt
	\advance\arrowsize by .5\pgflinewidth
	\pgfsetdash{}{0pt} 
	\pgfsetroundjoin   
	\pgfsetroundcap    
	\pgfpathmoveto{\pgfpoint{0\arrowsize}{0\arrowsize}}
	\pgfpatharc{-90}{-140}{4\arrowsize}
	\pgfusepathqstroke
	\pgfpathmoveto{\pgfpointorigin}
	\pgfpatharc{90}{140}{4\arrowsize}
	\pgfusepathqstroke
}

\tikzset{
    block/.style={
        circle, draw, fill=white
        },
    myarrow/.style={
        single arrow,  
        draw, 
        single arrow head extend=2mm, minimum width=30pt,
        },    
    myar/.style={
        rounded corners=2pt, fill=black!20, 
        },
    mytri/.style={
        isosceles triangle, anchor=apex,
        isosceles triangle apex angle=90,
        minimum width=50pt
        },
    }

\begin{document}

\doparttoc 
\faketableofcontents 

%

%
\runningauthor{Ignavier Ng, Shaoan Xie, Xinshuai Dong, Peter Spirtes, Kun Zhang}

\twocolumn[

\aistatstitle{Causal Representation Learning from General Environments under Nonparametric Mixing
}

\aistatsauthor{ Ignavier Ng$^{1}$ \And Shaoan Xie$^{1}$ \And Xinshuai Dong$^{1}$ \And Peter Spirtes$^{1}$ \And Kun Zhang$^{1,2}$}

\aistatsaddress{\hspace{-6em}$^{1}$Carnegie Mellon University\And \hspace{-6em}$^{2}$Mohamed bin Zayed University of Artificial Intelligence } ]

\begin{abstract}
\vspace{-0.3em}
Causal representation learning aims to recover the latent causal variables and their causal relations, typically represented by directed acyclic graphs (DAGs), from low-level observations such as image pixels. A prevailing line of research exploits multiple environments, which assume how data distributions change, including single-node interventions, coupled interventions, or hard interventions, or parametric constraints on the mixing function or the latent causal model, such as linearity. Despite the novelty and elegance of the results, they are often violated in real problems. Accordingly, we formalize a set of desiderata for causal representation learning that applies to a broader class of environments, referred to as general environments. Interestingly, we show that one can fully recover the latent DAG and identify the latent variables up to minor indeterminacies under a nonparametric mixing function and nonlinear latent causal models, such as additive (Gaussian) noise models or heteroscedastic noise models, by properly leveraging sufficient change conditions on the causal mechanisms up to third-order derivatives. These represent, to our knowledge, the first results to fully recover the latent DAG from general environments under nonparametric mixing. Notably, our results match or improve upon many existing works, but require less restrictive assumptions about changing environments.
\end{abstract}

\section{Introduction}\label{sec:introduction}
Causal representation learning (CRL) aims to recover the latent causal variables and their causal structure from observations of variables that might be non-causal \citep{scholkopf2021causal}. It is of great importance in 
many scientific fields, as in real-life complex systems, available measurements are often low-level, indirect, and high-dimensional, e.g., image pixels, linguistic tokens, and  and gene expressions.

Despite growing interest, CRL remains a highly challenging task. Even in the case of independent latent variables—commonly referred to as nonlinear independent component analysis (ICA)—the problem is difficult~\citep{hyvarinen2023nonlinear}. Nonlinear ICA, despite being a strictly easier task due to the lack of relationships among latent variables, is notoriously unidentifiable without additional assumptions, as different latent representations can explain the observed data equally well, yet may not align with the underlying data generating process \citep{hyvarinen1999nonlinear}. Furthermore, CRL inherits the complexities of causal discovery \citep{spirtes2001causation,glymour2019review}, which is already challenging even when all causal variables are fully observed.

For the task of  CRL, a line of work focuses on purely observational data with parametric and graphical assumptions. For instance, various graphical conditions have been proposed with linearity \citep{silva2006learning,cai2019triad,xie2020generalized,xie2022identification,adams2021identification,huang2022latent,dong2023versatile} or discrete assumptions \citep{kivva2021learning}. Another prevailing line of research, which is the focus of this work, addresses the problem by making use of data from multiple distributions/environments, where the change of distribution is often assumed to arise from hard interventions \citep{vonkugelgen2023nonparametric,jiang2023learning,bing2024identifying}, single-node interventions \citep{ahuja2023interventional,squires2023linear,jiang2023learning,varici2023score,zhang2023identifiability}, coupled interventions \citep{vonkugelgen2023nonparametric,jin2023learning}, or counterfactual views \citep{vonkgelgen2021self,brehmer2022weakly}. Other related works are further discussed in \cref{app:related_works}.\looseness=-1

\begin{table*}[t]
    \label{tab:related_works}
    \centering
    \caption{Comparison of of several existing identifiability results of multi-environment CRL based on soft interventions. We only outline the key assumptions and results, while omitting some additional assumptions required by several studies. For a more comprehensive comparison, including works that rely on hard interventions, refer to Table~\hyperref[tab:related_works_complete]{2} in the supplementary materials.}
    \vskip 0.1in
    \scalebox{0.92}{
    \begin{tabular}{cccccccc}
        \toprule
        Work  & \makecell{Latent\\SEM} & \makecell{Mixing\\Function}  & \makecell{Desiderata\\Satisfied?} & \makecell{Identifiability of\\Latent Variables} & \makecell{Identifiability of\\ Latent DAG} \\
        \midrule
        \citet{squires2023linear}  & Linear & Linear & 1,2 & Up to ancestors  & Transitive closure\\
        \hline
        \citet{zhang2023identifiability}  & Nonlinear & Polynomial & 1,2 & Up to ancestors  & Transitive closure\\
        \hline
        \multirow{2}{*}{\citet{varici2024scorebased}}  & General & Linear & 1,2 & Up to ancestors  & Transitive closure\\
          & Nonlinear & Linear & 1,2 & Up to surrounding parents  & Full\\
        \hline
        \citet{ahuja2023interventional} & Bounded RV & Polynomial & 1,2 & Full & Full\\
        \hline
        \multirow{2}{*}{\citet{jin2023learning}}  & General & \textbf{General} & 1 & Up to surrounding parents & Full\\
          & Linear & Linear & \textbf{1,2,3} & Up to surrounding parents  & Full\\
        \hline
        \citet{varici2024linear}  & General & Linear & \textbf{1,2,3} & Up to ancestors & Transitive closure\\
        \hline
        \citet{zhang2024causal}  & General & \textbf{General} & \textbf{1,2,3} & Up to intimate neighbors & Moral graph\\
        \hline
        \multirow{2}{*}{\textbf{Ours}}  & ANM & \textbf{General} & \textbf{1,2,3} & Up to surrounding parents & Full\\
          & HNM & \textbf{General} & \textbf{1,2,3} & Up to surrounding parents & Full\\
        \bottomrule
    \end{tabular}
    }
\end{table*}

Despite the novelty and elegance of the results, many of these assumptions regarding how distribution changes are often violated in real problems. For instance, in genomics, it is commonplace to encounter soft interventions instead of hard interventions, e.g., in RNA interference \citep{dominguez2015editing}. Accordingly, in this work, we formalize a set of desiderata to establish more realistic assumptions about changing environments for CRL, and we refer to the setting as \emph{CRL from general environments}, detailed in \cref{section:desiderata}. Under such a general and realistic CRL setting, we propose, to our knowledge, the first identifiability result to fully recover the latent DAG and identify the latent causal variables up to minor indeterminacies, while allowing for nonlinear latent causal models and nonparametric mixing. In contrast, the existing results for general environments either require linearity on the mixing function \citep{jin2023learning,varici2024linear}, or can only identify the latent DAG up to its moral graph \citep{zhang2024causal}. A comparison is provided in Table~\ref{tab:related_works}.

Concretely, we show that, with general environments, nonparametric mixing, and nonlinear latent causal models—such as additive (Gaussian) noise models (\cref{theorem:identifiability_anm}) or heteroscedastic noise models (\cref{theorem:identifiability_hnm})—it is possible to fully recover the latent DAG and identify the underlying latent variables up to minor indeterminacies. Specifically, each latent variable can be identified up to a function of itself and its \emph{surrounding parents}~\citep{jin2023learning} in the true latent DAG. While \citet{zhang2024causal} exploit sufficient change conditions on the causal mechanisms up to second-order derivatives to identify the moral graph, we introduce a fundamentally different approach that properly leverages third-order derivatives to extract causal ordering information of the latent variables, enabling us to fully recover the latent DAG. This leads to a novel set of proof techniques, which we hope will inspire future research in this area. Finally, we validate our identifiability theory through simulation studies.
\section{Problem Setting}\label{sec:problem_setting}
\paragraph{Setup.} We describe the problem setting of CRL. We assume that the observed random variables $X=(X_1,\dots,X_d)$ and latent variables $Z=(Z_1,\dots,Z_n)$ follow the data generating process below:
\begin{equation}\label{eq:data_generating_process}
\begin{alignedat}{3}
\hspace{-1.2em}\text{(Mixing)}  & \;\,\,\, X&&=g(Z),\\
\hspace{-1.2em}\text{(Latent SEM)} & \;\,\,\, Z_i &&= f_i^{(u)}(\textrm{PA}(Z_i;\mathcal{G}_Z), \epsilon_i^{(u)}), \,\,i\in[n].
\end{alignedat}
\end{equation}
Here, the observed random variables $X$ are generated from the latent variables $Z$ via an unknown, nonparametric mixing function $g$ that is assumed to be a $\mathcal{C}^2$-diffeomorphism onto its image $\mathcal{X}\subseteq \mathbb{R}^d$. In each environment indexed by $u$, the latent variables $Z$ follow a structural equation model (SEM), characterized the same but unknown DAG $\mathcal{G}_Z$ consisting of nodes $\{Z_i\}_{i=1}^n$. Here, $\textrm{PA}(Z_i; \mathcal{G}_Z)$ denotes the parents of $Z_i$ in DAG $\mathcal{G}_Z$. 

Denote by $p_Z^{(u)}$ and $p_X^{(u)}$ the probability density functions of $Z$ and $X$, respectively. To simplify notation and when the context is clear, we omit subscripts in the density functions and use $X$ and $Z$ to represent both the random variables and their specific values. We assume that $p_Z^{(u)}$ is third-order differentiable and has full support on $\mathbb{R}^n$.

A summary of the data generating process is illustrated in \cref{fig:setting}. In this work, we aim to estimate the latent variables $Z$ and the latent DAG $\mathcal{G}_Z$ up to minor indeterminacies from samples of $X$ across a number of environments, i.e., $u\in [m]$. The distribution changes across environments may arise from heterogeneous data or nonstationary time series. Here, we focus on changes that can be represented by interventions---whether hard (perfect) or soft (imperfect) \citep{eberhardt2013direct,yang2018characterizing}.

\begin{figure}[t]
\centering
	{\hspace{0cm}\begin{tikzpicture}[scale=.75, line width=0.5pt, inner sep=0.2mm, shorten >=.1pt, shorten <=.1pt]
		\draw (0, 0) node(3) [circle]  {{\footnotesize\,$\dots$\,}};
		\draw (-2, 0) node(2)[circle, draw]  {{\footnotesize\,$Z_2$\,}};
		\draw (2, 0) node(4)[circle, draw]  {{\footnotesize\,$Z_n$\,}};
		\draw (-4, 0) node(1)[circle, draw]  {{\footnotesize\,$Z_1$\,}};
		\draw (-1, 1.4) node(5)  [circle, draw,fill=gray!30]{{\footnotesize\,\,\,$u$\,\,\,}};
		\draw[-arcsq] (2) -- (1); 
		\draw[-arcsq] (2) -- (3); 
		\draw[-arcsq] (3) -- (4);
		\draw[-arcsq] (5) -- (1);
        \draw[-arcsq] (5) -- (2);
        \draw[-arcsq] (5) -- (3);
        \draw[-arcsq] (5) -- (4);
\node[rectangle, dashed, inner sep=1.5mm, draw=black!100, opacity=1, fit = (1) (4), yshift = 0mm](aa) {};
\node[myarrow, anchor=tip, minimum height=32pt, rotate=-90] at ([yshift=-43pt]aa.south) (bb) {$g$};
		\draw (-0.3, -2.82) node(8)[circle]{{\footnotesize\,$\dots$\,}};
        \draw (-1.7, -2.82) node(7)[circle, draw, fill=gray!30]{{\footnotesize\,$X_2$\,}};
        \draw (-3.1, -2.82) node(6)[circle, draw, fill=gray!30]{{\footnotesize\,$X_1$\,}};
        \draw (1.1, -2.82) node(9)[circle, draw, fill=gray!30]{{\footnotesize\,$X_d$\,}};
\node[rectangle, dashed, inner sep=1.5mm, draw=black!100, opacity=1, fit = (6) (9), yshift = 0mm](aa) {};
		\end{tikzpicture}}
	\caption{The observed random variables $X$ are generated from the latent variables $Z$ via an unknown, nonparametric mixing function $g$. The causal mechanism for each latent variable $Z_i$ may vary across different environments specified by $u$. The gray shading of the nodes indicates that the variables are observable.}
	\label{fig:setting}
 \vspace{-0.4em}
\end{figure}
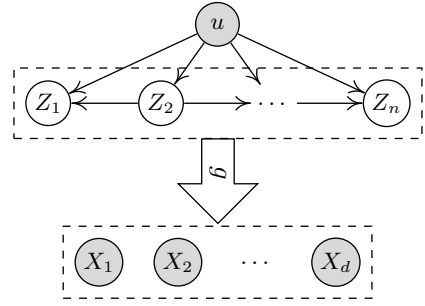

\paragraph{Observational equivalence.}
For identification, we consider a model $(\hat{g},p_{\hat{Z}},\mathcal{G}_{\hat{Z}})$ that follows the data generating process specified in \cref{eq:data_generating_process}, and generates a data distribution that matches the distribution of the observed random variables $X$ across the given environments, i.e.,
\begin{equation}\label{eq:observational_equivalence}
p^{(u)}_{\hat{X}}(x)=p^{(u)}_X(x), \quad \forall \, u\in [m], \, x\in\mathcal{X}^{(u)},
\end{equation}
where $\hat{Z}$ denote the estimated latent variables and $\hat{X}=\hat{g}(\hat{Z})$. In this case, the models $(\hat{g},p_{\hat{Z}},\mathcal{G}_{\hat{Z}})$ and $(g,p_Z,\mathcal{G}_Z)$ are said to be \emph{observationally equivalent}. In practice, observational equivalence can be attained by maximum likelihood estimation in the large sample limit, where various approaches including normalizing flows~\citep{danilo2015variational} can be applied.

\paragraph{Notations.}
We write $[n]\coloneqq \{1,\dots,n\}$, and denote by $\oplus$ the vector concatenation symbol. For a matrix $A$, we denote by $A_{i,:}$ and $A_{:,j}$ its $i$-th row and $j$-th column, respectively. For a vector $Z=(Z_1,\dots,Z_n)$, we write $Z_{[k]}\coloneqq(Z_1,\dots,Z_k)$, and similarly for a matrix. With a slight abuse of notation, we occasionally treat a vector as a set. For example, we may write $Z_j\in Z$ if $Z_j\in \{Z_i\}_{i=1}^n$. Similar to the cardinality of a set, we denote by $|Z|$ the number of entries in vector $Z$. Moreover, we denote by $\mathcal{S}(\mathcal{G}_Z)$ the set of sink nodes in DAG $\mathcal{G}_Z$. For two DAGs $\mathcal{G}_Z$ and $\mathcal{G}_{\hat{Z}_{\pi}}$ where $\pi$ is a permutation, we say that they are identical when $Z_i\rightarrow Z_j$ in $\mathcal{G}_Z$ if and only if $\hat{Z}_{\pi(i)}\rightarrow \hat{Z}_{\pi(j)}$ in $\mathcal{G}_{\hat{Z}}$. Following \citet{zhang2024causal}, we define the intimate neighbors of $Z_i$ as
\begin{flalign*}
\Psi(Z_i;\mathcal{M}_Z)\coloneqq &\{ Z_j \mid  Z_j,j\neq i, \textrm{ is adjacent to }Z_i\textrm{ and}\\
&\quad \textrm{all other neighbors of }Z_i\textrm{ in }\mathcal{M}_Z \}.
\end{flalign*}
\section{Desiderata for CRL from General Environments}
\label{section:desiderata}
Nonlinear ICA may be viewed as a special case of CRL in which the true latent DAG is an empty graph. However, even in this case, the latent variables has shown to be unidentifiable without any assumptions~\citep{hyvarinen1999nonlinear}. This challenge extends to CRL, a problem that is inherently more complex. 
As discussed in \cref{sec:introduction,app:related_works}, a prevailing line of research relies on distributional changes in the latent variables across different environments, including hard interventions, single-node interventions, or coupled interventions, some of which may be overly restrictive in practice.

In this work, we formalize a set of desiderata for CRL which applies to a broader class of environments that may be more realistic in real-world scenarios, referred to as general environments:
\begin{itemize}
\item \textbf{Desideratum~1:} The setting does not require hard interventions. The interventions can be either soft or hard.
\item \textbf{Desideratum~2:} The setting does not assume any prior knowledge of intervention targets or their coupling pattern across environments.
\item \textbf{Desideratum~3:} The interventions can be either single-node or multi-node across environments.
\end{itemize}
Note that the term ``general environments'' has been used by \citet{jin2023learning}; here, we aim to provide a precise formulation of the desiderata, and use that term to refer to environments that satisfy the above desiderata. For Desideratum~1, hard interventions eliminate the dependency between the targeted variables and their parents, while soft interventions modify the causal mechanism of targeted variables without fully eliminating the influences of their parents, which thus are often considered more realistic \citep{yang2018characterizing,jaber2020causal}. For instance, in genomics, it is commonplace to encounter soft interventions, e.g., in  CRISPR-mediated gene activation or RNA interference \citep{dominguez2015editing}. For Desideratum~3, in real-world environments, the causal mechanisms of multiple latent variables may simultaneously change, resulting in multi-node interventions that extend beyond the scope of single-node interventions.\looseness=-1

For Desideratum~2, since the latent variables are unknown by definition, it is in practice often infeasible to know the intervention targets and their coupling pattern (i.e., which environments share the same intervention targets), such as biology~\citep{squires2020permutation,tejada2023causal} and robotics~\citep{lee2023scale}. Furthermore, we show in \cref{proposition:coupled_single_node_interventions} that knowing the coupling pattern of the intervention targets is equivalent to knowing the intervention targets (up to variable permutation) for single-node interventions, with a proof given in \cref{app:proof_proposition_coupled_single_node_interventions}.
\begin{restatable}[Coupled single-node interventions]{proposition}{PropositionCoupledSingleNodeInterventions}\label{proposition:coupled_single_node_interventions}
Consider a set of single-node interventions on the latent variables $Z$. The knowledge of which interventions share the same targets is equivalent to the knowledge of intervention targets  up to variable permutation.
\end{restatable}

In Tables~\ref{tab:related_works} and~\hyperref[tab:related_works_complete]{2} in the supplementary materials, we specify which desiderata above each existing result satisfies. Surprisingly, many existing results fail to satisfy at least one of the desiderata, indicating that current approaches may impose limitations or restrictive assumptions that hinder their applicability in more general, realistic environments.

\section{CRL with Latent Additive Noise Models}\label{sec:anm}
In this section, we present our identifiability theory for general environments, which leverages specific properties of additive noise models (ANMs). Before discussing it, we first review the key idea of~\citet{zhang2024causal}. In particular, the authors utilize the following property~\citep{lin1997factorizing} regarding the conditional independence structure of the latent variables:
\begin{equation}
\label{eq:cross_de}
Z_i \independent Z_j \mid Z_{[n] \backslash \{i,j\}} \iff \frac{\partial^2 \log p^{(u)}(Z)}{\partial{Z_i} \partial{Z_j}} = 0.
\end{equation}
Denote by $\mathcal{M}_Z$ the Markov network over latent variables $Z$, where $\mathcal{M}_Z$ consists of nodes $\{Z_i\}_{i=1}^n$ and undirected edges $\{Z_i,Z_j\}\in\mathcal{E}(\mathcal{M}_Z)$ if and only if $Z_i \notindependent Z_j \mid Z_{[n] \backslash \{i,j\}}$. \cref{eq:cross_de} indicates that $Z_i$ and $Z_j$ are not adjacent in the Markov network if and only if the above equation holds. \citet{zhang2024causal} then utilize sufficient changes of the latent distribution, involving second-order derivatives above, to recover the Markov network (or moral graph under the faithfulness assumption~\citep{spirtes2001causation}) over $Z$ in a nonparametric setting. However, since this approach essentially exploits a certain type of conditional independence information (that is invariant across all environments) via second-order derivatives that are symmetric, it generally cannot infer all causal directions, and can at most recover a Markov equivalence class over the latent variables, similar to constraint-based causal discovery approaches such as PC \citep{spirtes1991pc}.

In this work, we consider two classes of latent causal models and show that, interestingly, the third-order derivatives can reveal information about causal directions. This allows us to fully recover the latent DAG, by properly leveraging sufficient change conditions on the distributions up to third-order derivatives. Specifically, we focus on latent ANMs~\citep{hoyer2009nonlinear,rolland2022score} of the form:
\begin{equation}\label{eq:anm}
Z_i = f_i^{(u)}(\textrm{PA}(Z_i;\mathcal{G}_Z))+\epsilon_i^{(u)},\quad i\in[n],
\end{equation}
where the noise term $\epsilon_i$ follows Gaussian distribution and function $f_i^{(u)}$ is third-order differentiable. In the context of causal discovery from observed variables, it has been shown that the above causal model is identifiable under certain conditions \citep{hoyer2009nonlinear,rolland2022score}. We further discuss the use of heteroscedastic noise models in \cref{sec:hnm}.

In \cref{sec:third_order_derivatives_anm}, we discuss the property of sink nodes. We then provide our identifiability result in \cref{sec:identifiability_anm} and the estimation method in \cref{sec:estimation_method_anm}.

\subsection{Third-Order Derivatives of Sink Nodes}\label{sec:third_order_derivatives_anm}
We show how sink nodes provide information related to the third-order derivatives of the distribution of latent variables. Specifically, each sink node has the following property. The proof is given in \cref{app:proof_lemmas_zero_derivative_anm}.
\begin{restatable}{lemma}{LemmaZeroDerivativeANM}\label{lemma:zero_derivative_anm}
Consider the data generating process in \cref{eq:anm} and let $Z_i$ be a sink node in DAG $\mathcal{G}_Z$. Then,
\[
\frac{\partial \log p^{(u)}(Z)}{\partial Z_i^2 \partial Z_j}=0 \quad\textrm{for}\quad j\in[d].
\]
\end{restatable}
As discussed in \cref{sec:identifiability_anm}, this property provides information that can be leveraged to identify the sink node from the latent Markov network. Once the sink node is identified, the Markov network property in \cref{eq:cross_de} can be used to infer its parents. Specifically, the parents of a sink node $Z_i$ are precisely the neighboring nodes of $Z_i$ in the Markov network over $Z$.

It is worth noting that the above lemma closely resembles the work of \citet{rolland2022score}, who developed a method for estimating causal structures among observed variables using score matching. However, they focus on causal discovery in settings where all variables are observed, and utilize the information that the \emph{variance} of the second-order derivatives is zero. This limits its direct applicability to CRL. In contrast, as we will demonstrate in the next section, leveraging third-order derivatives as in \cref{lemma:zero_derivative_anm} allows us to establish sufficient change conditions on the latent distribution, enabling identification of the latent DAG.

Moreover, in \cref{lemma:zero_derivative_anm}, we leverage the fact that sink node is a sufficient condition for the corresponding third-order partial derivative to be equal to zero. To establish it as a necessary condition, one could adopt similar assumptions to those in \citet{rolland2022score}, such as requiring $f_i^{(u)}$ to be nonlinear in every component. It is worth noting that our results do not rely on this necessity, as long as \cref{assumption:sufficient_changes_markov_network,assumption:sufficient_changes_intimate_neighbors_anm} can be satisfied.

\subsection{Identifiability Theory}\label{sec:identifiability_anm}
We introduce our identifiability result for general environments and nonparametric mixing. The key idea is to properly leverage sufficient change conditions on the causal mechanisms up to third-order derivatives, building upon \cref{lemma:zero_derivative_anm}.

We first provide the definition of \emph{surrounding parents}~\citep{jin2023learning,varici2024scorebased} that is crucial for the identifiability result of the estimated latent variables:\footnote{
The published version of this paper used a more complex definition of ``intimate parents''. This definition can be simplified (without affecting correctness) and becomes equivalent to the notion of surrounding parents~\citep{jin2023learning,varici2024scorebased}, which is adopted in this updated version. The earlier claim that our result is stronger than \citet{jin2023learning} is also removed.}
\begin{flalign*}
&\operatorname{sur}(Z_i;\mathcal{G}_Z)\coloneqq  \{Z_j\in\textrm{PA}(Z_i;\mathcal{G}_Z) \mid  Z_j\textrm{ is a parent of}\\
&\qquad\qquad \textrm{every child of }Z_i\textrm{ in } \mathcal{G}_Z\}.
\end{flalign*}
Specifically, \citet{zhang2024causal} showed that the underlying Markov network (or moral graph under faithfulness assumption) $\mathcal{M}_Z$ can be recovered, and that the latent variables can be identified up to permutation and mixing with intimate neighbors, i.e., $Z_i\cup \Psi(Z_i;\mathcal{M}_Z)$. In contrast, with the ANM described in \cref{eq:anm}, we show that, one can fully recover the latent DAG $\mathcal{G}_Z$ and identify the latent variables up to permutation and mixing with surrounding parents, i.e., $Z_i\cup \operatorname{sur}(Z_i;\mathcal{G}_Z)$. Importantly, we have $\operatorname{sur}(Z_i;\mathcal{G}_Z)\subseteq \Psi(Z_i;\mathcal{M}_Z)$, indicating that our identifiability result significantly refines that of \citet{zhang2024causal}. Whereas their indeterminacy of latent variables involves specific types of parents, children, and spouses, our result restricts the ambiguity to only a specific type of parents.

We now present the assumptions and the identifiability result, with a detailed explanation of both provided later in this section.
\begin{assumption}[Sufficient changes for Markov network]\label{assumption:sufficient_changes_markov_network}
Let $Z'$ denote any subset of latent variables $Z$ that includes all of their ancestors, and $n'\coloneqq|Z'|$. For each value of $Z'$, there exist $2n'+|\mathcal{M}_{Z'}|+1$ values of $u$, i.e., $u_j$ with $j=0,\dots,2n'+|\mathcal{M}_{Z'}|$, such that the vectors $\tau(Z',u_j)-\tau(Z',u_0)$ with $j=1,\dots,2n'+|\mathcal{M}_{Z'}|$ are linearly independent, where vector $\tau(Z', u)$ is defined as:
\begin{flalign*}
& \tau(Z', u)=\left(\frac{\partial \log p^{(u)}(Z')}{\partial  Z_i}\right)_{Z_i\in Z'}\\
&\qquad\qquad\qquad\oplus\left(\frac{\partial^2 \log p^{(u)}(Z')}{\partial  Z_i^2}\right)_{Z_i\in Z'}\\
&\qquad\qquad\qquad \oplus \left(\frac{\partial^2 \log p^{(u)}(Z')}{\partial  Z_i \partial  Z_j}\right)_{i<j\cl\{Z_i,Z_j\}\in\mathcal{E}(\mathcal{M}_{Z'})}.
\end{flalign*}
\end{assumption}
\begin{assumption}[Sufficient changes for sink nodes]\label{assumption:sufficient_changes_intimate_neighbors_anm}
Let $Z'$ denote any subset of latent variables $Z$ that includes all of their ancestors. For each value of $Z'$, there exist $L+1$ values of $u$, i.e., $u_j$ with $j=0,\dots,L$, such that the vectors $w(Z',u_j)-w(Z',u_0)$ with $j=1,\dots,L$ are linearly independent, where $L\coloneqq |w(Z', u)|$ and vector $w(Z', u)$ is defined as:
\begin{flalign*}
& w(Z', u)=\left(\frac{\partial \log p^{(u)}(Z')}{\partial  Z_i}\right)_{Z_i\in Z'}\\
& \qquad\qquad\oplus \left(\frac{\partial^2 \log p^{(u)}(Z')}{\partial  Z_i^2}\right)_{Z_i\in Z'}\\
& \qquad\qquad\oplus \left(\frac{\partial^2 \log p^{(u)}(Z')}{\partial  Z_i \partial  Z_j}\right)_{i<j\cl\{Z_i,Z_j\}\in\mathcal{E}(\mathcal{M}_{Z'})}\\
& \qquad\qquad\oplus \left(\frac{\partial^3 \log p^{(u)}(Z')}{\partial  Z_i^3}\right)_{Z_i\in Z'\setminus \mathcal{S}(\mathcal{G}_{Z'})}\\
& \qquad\qquad\oplus \left(\frac{\partial^3 \log p^{(u)}(Z')}{\partial  Z_i^2 \partial  Z_j}\right)\mathstrut_{\begin{subarray}{l} \vspace{0.2em}\\i<j\cl\{Z_i,Z_j\}\in\mathcal{E}(\mathcal{M}_{Z'}),\\ Z_i\not\in\mathcal{S}(\mathcal{G}_{Z'}) \end{subarray}}\\
& \qquad\qquad \oplus \left(\frac{\partial^3 \log p^{(u)}(Z')}{\partial  Z_i \partial  Z_j\partial Z_k}\right)\mathstrut_{\begin{subarray}{l} \vspace{0.2em}\\i<j<k\cl \{Z_i,Z_j\},\\\{Z_j,Z_k\},\{Z_i,Z_k\}\in\mathcal{E}(\mathcal{M}_{Z'})\end{subarray}}.
\end{flalign*}
\end{assumption}

\begin{restatable}[Identifiability with latent ANMs]{theorem}{ThmIdentifiabilityANM}\label{theorem:identifiability_anm}
Consider the data generating process in \cref{eq:data_generating_process,eq:anm}. Suppose that \cref{assumption:sufficient_changes_markov_network,assumption:sufficient_changes_intimate_neighbors_anm}, as well as the faithfulness assumption, hold. Let $\mathcal{G}_{\hat{Z}}$ and $\hat{Z}$ be the output of \cref{alg:crl}. Then, there exists a permutation $\pi$ of the estimated latent variables $\hat{Z}$, denoted as $\hat{Z}_{\pi}$, such that:
\begin{enumerate}
\item (Identifiability of $\mathcal{G}_Z$) $\mathcal{G}_{\hat{Z}_\pi}$ and $\mathcal{G}_Z$ are identical.
\item (Identifiability of $Z$) $\hat{Z}_{\pi(i)}$ is solely a function of a subset of $\{Z_i\}\cup\operatorname{sur}(Z_i;\mathcal{G}_Z)$.
\end{enumerate}
\end{restatable}

\begin{algorithm}[t]
	\caption{Iterative identification of Markov network and sink nodes}
	\label{alg:crl}
	\hspace*{0.02in} {\bf Input:}
	$p^{(u)}(X),u=1,\dots,m$\\
	\hspace*{0.02in} {\bf Output:}
	$\mathcal{G}_{\hat{Z}}$ and $\hat{Z}$

 \begin{algorithmic}[1]
    \FOR{$t=1,\dots,n-1$}
        \STATE Learn $(\hat{g}^t,p_{\hat{Z}}^t,\mathcal{G}_{\hat{Z}}^t)$ to achieve \cref{eq:observational_equivalence} with minimal number of edges for Markov network $\mathcal{M}^{t}_{\hat{Z}_{[n+1-t]}}$, while (i) constraining the local structure (including incoming and outgoing edges) of $\{\hat{Z}_i\}_{i=n+2-t}^n$ in $\mathcal{G}_{\hat{Z}}^t$ to be identical to those in $\mathcal{G}_{\hat{Z}}^{t-1}$, and (ii) ensuring $\mathcal{M}^{t}_{\hat{Z}_{[k]}}\subseteq \mathcal{M}^{t-1}_{\hat{Z}_{[k]}}$ for $k=n+2-t,\dots,n$.
        \STATE Let $\hat{Z}_i$ a sink node in DAG $\mathcal{G}_{\hat{Z}_{[n+1-t]}}^t$. Reorder $\hat{Z}$ in $(\hat{g}^t,p_{\hat{Z}}^t,\mathcal{G}_{\hat{Z}}^t)$ by swapping $\hat{Z}_i$ and $\hat{Z}_{n+1-t}$.
    \ENDFOR
	\STATE {\bfseries return} $\mathcal{G}^{n-1}_{\hat{Z}}$ and  $(\hat{g}^{n-1})^{-1}(X)$
	\end{algorithmic}
\end{algorithm}

The proof of the theorem is provided in \cref{app:theorem_identifiability_anm}, which makes use of \cref{lemma:zero_derivative_anm}. Here, the word `solely' implies that $\hat{Z}_{\pi(i)}$ is not a function of other latent variables $Z_j$ not specified above. We now outline the key intuitions behind the assumptions and the identifiability result. Intuitively speaking, \cref{assumption:sufficient_changes_markov_network,assumption:sufficient_changes_intimate_neighbors_anm} impose certain type of sufficient change conditions on the distribution of latent variables. Note that different forms of sufficient change conditions have been adopted in nonlinear ICA~\citep{hyvarinen2023nonlinear} and CRL~\citep{zhang2024causal}.

First, under \cref{assumption:sufficient_changes_markov_network}, observational equivalence, together with a sparsity constraint on the recovered Markov network, allows us to recover the true Markov network (or moral graph under faithfulness assumption) up to isomorphism and to identify the latent variables up to permutation and mixing with their intimate neighbors, similar to the result of \citet{zhang2024causal}. Next, going beyond Markov network, \cref{assumption:sufficient_changes_intimate_neighbors_anm} enables us to further identify the sink nodes of the DAG, whose parents are precisely their corresponding neighbors in the recovered Markov network. At this stage, having identified one of the sink nodes and its parents, we can fix the local structure of the sink node, and proceed to learn the edges among the remaining nodes, while ensuring that the Markov network (or moral graph) remains consistent with the previous iterations.\footnote{For clarity, in Lines 2 and 3 of \cref{alg:crl}, we handle only one sink node per iteration, even though multiple sink nodes may exist.}  This iterative procedure allows us to progressively identify sink nodes and their local structure at each step, and move to the next stage. Ultimately, this process leads to the identification of the entire DAG, as well as the latent variables up to mixing with their surrounding parents.\looseness=-1

The resulting method is described in \cref{alg:crl}. Due to the permutation indeterminacy of the latent variables, we design \cref{alg:crl} to output $\hat{Z}$ in a causal order with respect to the latent DAG $\mathcal{G}_{\hat{Z}}$. It is worth noting that the overall idea is analogous to ordering-based causal discovery methods~\citep{raskutti2018learning,ghoshal2018learning,rolland2022score}, which aim to identify the causal order assuming that all causal variables are observed. However, since the relevant causal variables are unobserved in our setting, our result requires careful reasoning over the latent variables and latent DAG, leading to entirely different assumptions and algorithmic procedure.

\begin{example}[Identifiability result]
Consider the true latent DAG $\mathcal{G}_Z:Z_1\rightarrow Z_2\leftarrow Z_3$. Applying \cref{theorem:identifiability_anm}, there exist a permutation $\pi$ of the estimated latent variables $\hat{Z}$, such that $\mathcal{G}_{\hat{Z}_\pi}$ and $\mathcal{G}_Z$ are identical. Also, we have: (a) $\hat{Z}_{\pi(1)}$ is solely a function of $Z_1$, (b) $\hat{Z}_{\pi(2)}$ is solely a function of $\{Z_1,Z_2,Z_3\}$, and (c) $\hat{Z}_{\pi(3)}$ is solely a function of $Z_3$. Here, both $Z_1$ and $Z_3$ are identified up to component-wise transformation.
\end{example}
In the example above, the identifiability theory by \citet{zhang2024causal} can only recover the moral graph of $\mathcal{G}_Z$, i.e., an undirected triangle. Moreover, in their result, each estimated latent variable $\hat{Z}_{\pi(i)}$ can still be a function of all latent variables $\{Z_1,Z_2,Z_3\}$, indicating that no disentanglement is achieved. In contrast, our identifiability theory significantly advances upon their result by providing a stronger recovery of the latent causal structure and latent variables.

Additionally, \citet{jin2023learning} investigated the setting of single-node and coupled soft interventions, showing that the true latent DAG can be fully recovered, with the latent variables identified up to permutation and mixing with surrounding parents. We obtain the same level of identifiability as theirs, despite not requiring either single-node interventions or coupled interventions. 

\paragraph{Number of environments.} \cref{assumption:sufficient_changes_markov_network,assumption:sufficient_changes_intimate_neighbors_anm} require multiple environments to satisfy the assumption of linear independence across partial derivatives of latent distributions. Denote by $|\Delta(\mathcal{M}_Z)|$ be the number of triangles in the Markov network (or moral graph under the faithfulness assumption) over latent variables $Z$. The number of environments we need is (at least) $3n+3|\mathcal{E}(\mathcal{M}_Z)|+|\Delta(\mathcal{M}_Z)|-2|\mathcal{S}(\mathcal{G}_Z)|+1$. Note that \citet{zhang2024causal} require (at least) $2n+|\mathcal{E}(\mathcal{M}_Z)|+1$ environments for recovering the latent Markov network, while we leverage more environments to recover the latent DAG. 

\paragraph{Discussion of assumptions.}
Intuitively speaking, \cref{assumption:sufficient_changes_markov_network,assumption:sufficient_changes_intimate_neighbors_anm} requires that the causal mechanisms among the latent variables change sufficiently across different environments. Such distribution changes, together with the invariant mixing function, provide useful information for us to recover the latent variables and their causal relations.

To illustrate, consider a model with three latent variables that follow the SEM: $Z_1 = \epsilon_1^{(u)}$, $Z_2 = f_2^{(u)}(Z_1)+\epsilon_3^{(u)},$ and $Z_3 = f_3^{(u)}(Z_2)+\epsilon_2^{(u)}$, where functions $f_2^{(u)}$ and $f_3^{(u)}$ are parameterized by multilayer perceptrons (MLPs) (which can approximate any function). Here, the latent DAG is $Z_1\rightarrow Z_2\rightarrow Z_3$. Suppose that the variances of the noise terms $\epsilon_i^{(u)}$ and the weights of the MLPs are generated randomly for different environment $u$. Then, the sufficient change assumptions will hold almost surely.

Furthermore, \cref{assumption:sufficient_changes_markov_network,assumption:sufficient_changes_intimate_neighbors_anm} support both soft and hard interventions, as well as single-node or multi-node interventions. These assumptions also do not assume prior knowledge of intervention targets or their coupling patterns. Therefore, the desiderata described in \cref{section:desiderata} can be satisfied.

To discuss some implications, our assumptions imply that the functions $f_i^{(u)},i\in[n]$ cannot be linear in all environments. Otherwise, all third-order derivatives involved in \cref{assumption:sufficient_changes_intimate_neighbors_anm} will vanish across environments, rendering the assumption unsatisfiable. Loosely speaking, this suggests that nonlinear functions may sometimes offer greater potential for satisfying such assumptions.

\paragraph{With nonparametric mixing.}
Existing results addressing nonparametric mixing often assume that the coupling pattern of single-node interventions is known, i.e., which interventions target the same latent variables~\citep{vonkugelgen2023nonparametric,jin2023learning}. This assumption arises in part because their proofs rely primarily on first-order derivatives of the distribution of latent variables, necessitating knowledge of which causal mechanisms change together to establish sufficient changes on the distributions of latent variables. \citet{zhang2024causal} improved on this by leveraging second-order derivatives to recover the Markov network without requiring knowledge of the coupling pattern. In this work, we utilize third-order derivatives to achieve stronger identifiability results, allowing us to fully recover the latent DAG without assuming prior knowledge of the coupling pattern.\looseness=-1

\subsection{Estimation Method}\label{sec:estimation_method_anm}
Building on the identifiability result in \cref{theorem:identifiability_anm}, we now develop a practical method for \cref{alg:crl}. Specifically, Line 2 of \cref{alg:crl} requires achieving observational equivalence in \cref{eq:observational_equivalence}, which involves maximum likelihood estimation. Here, we use variational autoencoder (VAE) \citep{kingma2014autoencoding}, following previous works \citep{khemakhem2020variational,zhang2024causal}. Other estimation methods can also be adopted, such as normalizing flows~\citep{danilo2015variational,dinh2017density}.

Specifically, for Line 2 of \cref{alg:crl}, we maximize the marginal likelihood for each environment as 
\begin{align*}
    &\log p^{(u)}(X) \\
    &= \log \int_{\hat{Z}} p^{(u)}(X|\hat{Z})p^{(u)}(\hat{Z})d\hat{Z}\\ \nonumber
    &=\log \int_{\hat{Z}} \frac{q^{(u)}(\hat{Z}|X)}{q^{(u)}(\hat{Z}|X)}p^{(u)}(X|\hat{Z})p^{(u)}(\hat{Z})d\hat{Z}\\ \nonumber
    &\geq \mathbb{E}_{q^{(u)}(\hat{Z}\mid X)}(\log p^{(u)}(X|\hat{Z}))-D_{KL}(q^{(u)}(\hat{Z}|X), p^{(u)}(\hat{Z}))\\
    &=-\mathcal{L}_{\text{elbo}},
\end{align*}
where second term denotes the Kullback–Leibler (KL) divergence between between the prior distribution $p^{(u)}(\hat{Z})$ and the approximate posterior distribution $q^{(u)}(\hat{Z}|X)$. Using standard proof technique \citep{khemakhem2020variational,lachapelle2024nonparametric}, we can show that if the posterior $q^{(u)}(\hat{Z}|X)$ has enough capacity to express the true posterior, then Eq. (2) can be satisfied. Following existing works, we assume that the posterior $q^{(u)}(\hat{Z}|X)$ is a multivariate Gaussian distribution, where its mean and diagonal covariance are output by an encoder that is modeled as a MLP. The decoder, also modeled as an MLP, outputs the mean of $p^{(u)}(X|\hat{Z})$ that is assumed to be a multivariate Gaussian distribution, with a pre-specified variance, and thus the first term reduces to the reconstruction error of data from estimated latent variables $\hat{Z}$. Note that this decoder corresponds to the mixing function $\hat{g}$.

Now we need to construct a proper prior distribution $p^{(u)}(\hat{Z})$ to recover the dependent latent variables. Following \cref{eq:anm}, we model $\log p^{(u)}(\hat{Z})$ using the Gaussian log-likelihood:
\begin{flalign*}
\log p^{(u)}(\hat{Z})=&-\frac{1}{2}\sum_{i=1}^n \left(\frac{\hat{Z}_i- \hat{f}^{(u)}(\hat{A}_{i,:} \hat{Z})-\hat{\mu}_i^{(u)}}{\hat{\sigma}_i^{(u)}}\right)^2\\
&\qquad -\frac{1}{2}\sum_{i=1}^n \log(\hat{\sigma}_i^{(u)})^2,
\end{flalign*}
where where $\hat{f}_i^{(u)}$ is constructed as a three-layer MLP and $\hat{A}$ is a learnable binary adjacency matrix that represents the estimated DAG $\mathcal{G}_{\hat{Z}}$. Specifically, $\hat{A}_{i,j}=1$ indicates $\hat{Z}_j\rightarrow \hat{Z}_i$. We restrict $\hat{A}$ to be strictly lower triangular to avoid cyclic causal relations. Furthermore, we learn $\hat{\mu}_i^{(u)}$ and $\hat{\sigma}_i^{(u)}$ via $\hat{\mu}^{(u)} = h_i^{\mu}(u)$ and $\hat{\sigma}_i^{(u)} = h_i^{\sigma}(u)$, where $h_i^{\mu}$ and $h_i^{\sigma}$ are MLPs. Furthermore, in \cref{alg:crl}, we can find the sink nodes according to binary adjacency matrix $\hat{A}$, and constrain the local structure (which corresponds to constraint~(i) in Line 2 of \cref{alg:crl}) by fixing the corresponding parameters in the adjacency matrix $\hat{A}$ to certain binary values and further restricting them to be non-learnable during the training process.

We now discuss how to handle constraint (ii) in Line 2 of \cref{alg:crl}. Denote by $A^{t-1}$ the binary adjacency matrix estimated in the previous (i.e., $(t-1)$-th) iteration. Let $\mathcal{L}_\text{moral}$ be defined as
\begin{flalign*}
\sum_{k=n+2-t}^{n}&\big\|(I + \hat{A}_{[k],[k]})^T (I + \hat{A}_{[k],[k]})\\
&\qquad -(I + A^{t-1}_{[k],[k]})^T (I + A^{t-1}_{[k],[k]})\big\|_F^2.
\end{flalign*}
Instead of a hard constraint $\mathcal{L}_\text{moral}=0$ which can be enforced by constrained optimization methods, we simply treat it as a regularization term. Moreover, to enforce sparsity, we impose a constraint on the moral graph via the following regularization term, i.e., $\mathcal{L}_\text{sparsity}=\|(I + \hat{A})^T (I + \hat{A})\|_1$. The full objective function becomes $\mathcal{L}_\text{elbo}+\lambda_1\mathcal{L}_\text{sparsity}+\lambda_2\mathcal{L}_\text{moral}$.

After \cref{alg:crl} ends, the outputs of the encoder correspond to the estimated latent variables $\hat{Z}$ from observations $X$, while the binary matrix $\hat{A}$ represents the recovered DAG $\mathcal{G}_{\hat{Z}}$. In our experiments, we use a simplified version of the estimation approach by setting $\lambda_2=0$, which is found to also work well in practice.\looseness=-1

\section{CRL with Latent Heteroscedastic Noise Models}\label{sec:hnm}
After presenting the identifiability result for ANMs in \cref{sec:anm}, we now consider heteroscedastic noise models (HNMs)~\citep{xu2022causal,immer2023identifiability} that may be more general, where the noise terms are not assumed to have constant variances:
\begin{equation}\label{eq:hnm}
Z_i = f_i^{(u)}(\textrm{PA}(Z_i;\mathcal{G}_Z))+\sigma_i^{(u)}(\textrm{PA}(Z_i;\mathcal{G}_Z))\epsilon_i^{(u)},\,\,\,\,\, i\in[n].
\end{equation}
Here, the noise term $\epsilon_i$ follows Gaussian distribution, and functions $f_i^{(u)}$ and $\sigma_i^{(u)}$ are third-order differentiable. In the context of causal discovery from observed variables, it has been shown that the above causal model is identifiable under certain conditions \citep{xu2022causal,immer2023identifiability}. The estimation method for HNMs is nearly identical to that for ANMs in \cref{sec:estimation_method_anm}, which we omit here.

\begin{figure*}
    \centering
   \begin{tabular}{cc}
      \includegraphics[width=0.49\linewidth]{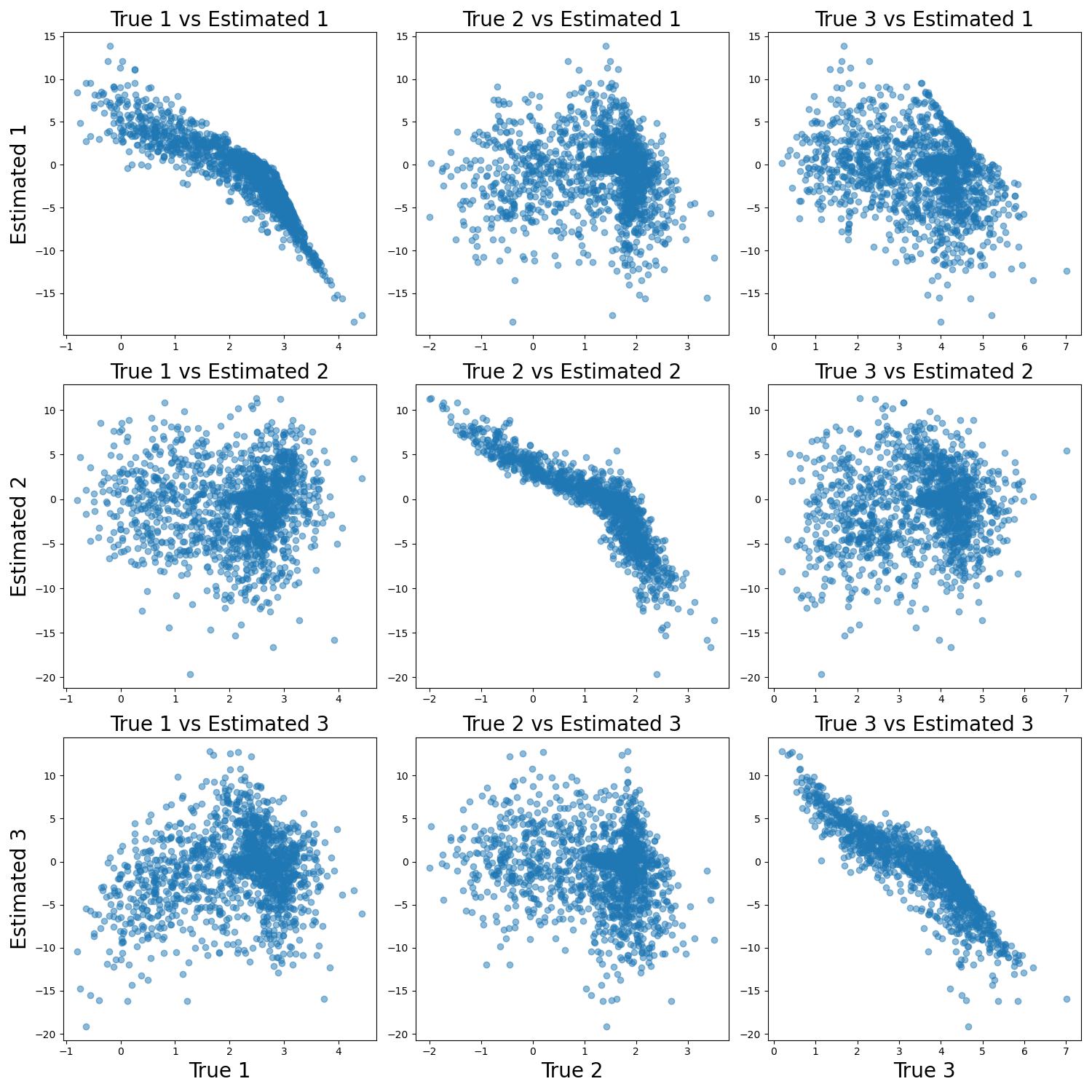}
       &\includegraphics[width=0.49\linewidth]{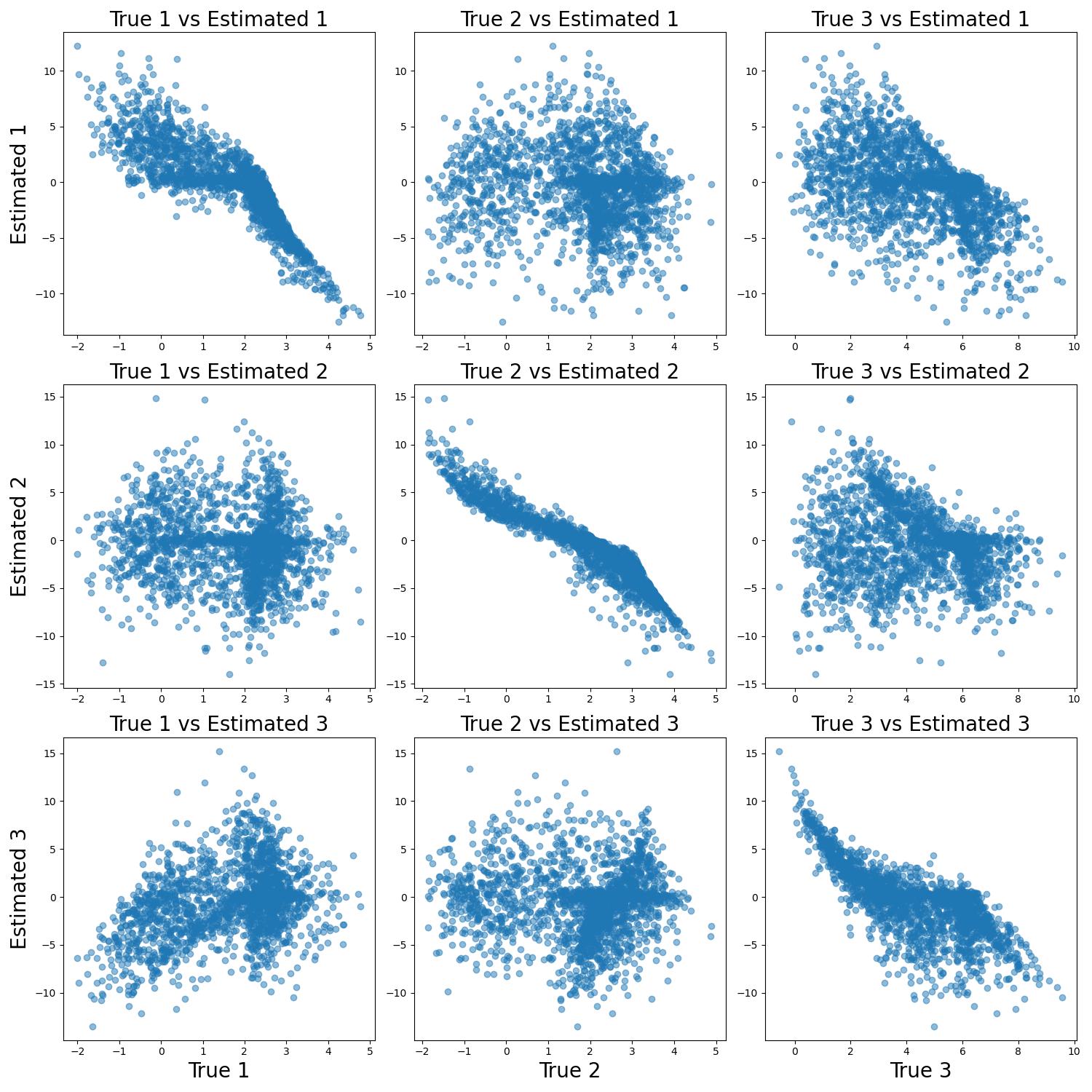} 
   \end{tabular}
    \caption{The recovered latent variables $\hat{Z}$ versus the true latent variables $Z$ under two latent DAGs: (1) $Z_1\rightarrow Z_2 \rightarrow Z_3$ and (2) $Z_1, Z_2\rightarrow Z_3$. 
    }
    \label{fig:simulation_results}
\end{figure*}

\subsection{Third-Order Derivatives of Sink Nodes}
Analogous to \cref{lemma:zero_derivative_anm} for ANMs, we present a corresponding result for HNMs. Specifically, the third-order derivatives of the sink node exhibit a property that provides valuable information for inferring the latent DAG. The proof is provided in \cref{app:proof_lemmas_zero_derivative_hnm}.
\begin{restatable}{lemma}{LemmaZeroDerivativeHNM}\label{lemma:zero_derivative_hnm}
Consider the data generating process in \cref{eq:hnm} and let $Z_i$ be a sink node in DAG $\mathcal{G}_Z$. Then,
\[
\frac{\partial \log p^{(u)}(Z)}{\partial Z_i^2 \partial Z_j}=0 \quad\textrm{for}\quad Z_j\not\in \textrm{PA}(Z_i;\mathcal{G}_Z).
\]
\end{restatable}

\subsection{Identifiability Theory}
We now present the identifiability theory for latent HNMs in general environments under nonparametric mixing. As with the setting of ANMs, the key idea is to properly leverage sufficient change conditions on the causal mechanisms of latent variables up to third-order derivatives, building on \cref{lemma:zero_derivative_hnm}.

The assumptions and results are presented below, with a proof provided in \cref{app:theorem_identifiability_hnm}. The assumptions and results here closely mirror those for ANMs, and we refer readers to \cref{sec:identifiability_anm} for a detailed explanation. 
\begin{assumption}[Sufficient changes for sink nodes]\label{assumption:sufficient_changes_intimate_neighbors_hnm}
Let $Z'$ denote any subset of latent variables $Z$ that includes all of their ancestors. For each value of $Z'$, there exist $L+1$ values of $u$, i.e., $u_j$ with $j=0,\dots,L$, such that the vectors $w(Z',u_j)-w(Z',u_0)$ with $j=1,\dots,L$ are linearly independent, where $L\coloneqq |w(Z', u)|$ and vector $w(Z', u)$ is defined as:
\begin{flalign*}
& w(Z', u)=\left(\frac{\partial \log p^{(u)}(Z')}{\partial  Z_i}\right)_{Z_i\in Z'}\\
& \qquad\qquad\oplus \left(\frac{\partial^2 \log p^{(u)}(Z')}{\partial  Z_i^2}\right)_{Z_i\in Z'}\\
& \qquad\qquad\oplus \left(\frac{\partial^2 \log p^{(u)}(Z')}{\partial  Z_i \partial  Z_j}\right)_{i<j\cl\{Z_i,Z_j\}\in\mathcal{E}(\mathcal{M}_{Z'})}\\
& \qquad\qquad\oplus \left(\frac{\partial^3 \log p^{(u)}(Z')}{\partial  Z_i^3}\right)_{Z_i\in Z'\setminus \mathcal{S}(\mathcal{G}_{Z'})}\\
& \qquad\qquad\oplus \left(\frac{\partial^3 \log p^{(u)}(Z')}{\partial  Z_i^2 \partial  Z_j}\right)_{i<j\cl\{Z_i,Z_j\}\in\mathcal{E}(\mathcal{M}_{Z'})}\\
& \qquad\qquad \oplus \left(\frac{\partial^3 \log p^{(u)}(Z')}{\partial  Z_i \partial  Z_j\partial Z_k}\right)\mathstrut_{\begin{subarray}{l} \vspace{0.2em}\\i<j<k\cl \{Z_i,Z_j\},\\\{Z_j,Z_k\},\{Z_i,Z_k\}\in\mathcal{E}(\mathcal{M}_{Z'})\end{subarray}}.
\end{flalign*}
\end{assumption}

\begin{restatable}[Identifiability with latent HNMs]{theorem}{ThmIdentifiabilityHNM}\label{theorem:identifiability_hnm}
Consider the data generating process in \cref{eq:data_generating_process,eq:hnm}. Suppose that \cref{assumption:sufficient_changes_markov_network,assumption:sufficient_changes_intimate_neighbors_hnm}, as well as the faithfulness assumption, hold. Let $\mathcal{G}_{\hat{Z}}$ and $\hat{Z}$ be the output of \cref{alg:crl}. Then, there exists a permutation $\pi$ of the estimated latent variables $\hat{Z}$, denoted as $\hat{Z}_{\pi}$, such that:
\begin{enumerate}
\item (Identifiability of $\mathcal{G}_Z$) $\mathcal{G}_{\hat{Z}_\pi}$ and $\mathcal{G}_Z$ are identical.
\item (Identifiability of $Z$) $\hat{Z}_{\pi(i)}$ is solely a function of a subset of $\{Z_i\}\cup\operatorname{sur}(Z_i;\mathcal{G}_Z)$.
\end{enumerate}
\end{restatable}
\section{Simulation Studies}\label{sec:experiments}
We conduct simulation studies to verify our identifiability theory. We consider three latent variables that follow our predefined DAG and the data generating procedure in \cref{eq:anm}. The mean and variances of $\epsilon_i^{(u)}$ are randomly sampled across different environments, and $f_i^{(u)}$ is modeled as a two-layer MLP that consists of weights randomly sampled from $\operatorname{Unif}[-2,2]$ and LeakyReLU activation function. For the mixing function, we employ a two-layer MLP to transform the latent variables $Z$ into observed variables $X$, ensuring injectivity by constraining the weight matrix to be orthogonal and using the LeakyReLU activation function. It is worth noting that the data generating procedure considered here satisfies the desiderata described in \cref{section:desiderata}. During the estimation process, we employ three-layer MLPs as the encoder and decoder of VAE. We use the Adam optimizer \citep{kingma2014adam} to train the model with learning rate of $0.001$.\looseness=-1

The results are shown in \cref{fig:simulation_results}, where each subfigure is a $3\times 3$ panel. The panel on $i$-th row and $j$-th column displays the relationship between the estimated latent variable $\hat{Z}_i$ and the true latent variable $Z_j$. Our method successfully recovers the correct structure in both cases. Notably, there is a clear one-to-one correspondence between $Z_1$ and $\hat{Z}_1$, as well as between $Z_2$ and $\hat{Z}_2$, indicating that these variables are identified up to component-wise transformations. The correspondence between $Z_3$ and $\hat{Z}_3$ is less pronounced, which is expected since $\hat{Z}_3$ is recovered as a function of both $Z_3$ and its surrounding parent $Z_2$. These observations validate our identifiability theory in \cref{theorem:identifiability_anm}. Further empirical studies on selecting number of latent variables and the hyperparameters are provided in \cref{app:select_number_latent,app:select_sparsity}, respectively.
\section{Conclusion}\label{sec:conclusion}
We formalize a set of desiderata to establish more realistic assumptions about changing environments in CRL, which we refer to as general environments. Under this setting, we show that it is possible to fully recover the latent DAG and identify the latent variables up to minor indeterminacies, while allowing for nonparametric mixing and nonlinear latent causal models, such as ANMs and HNMs with Gaussian noise. Our results properly leverage sufficient change conditions on the causal mechanisms up to third-order derivatives, to extract causal ordering information and fully recover the latent DAG. We validate our identifiability theory through simulation studies. Future works include applying the proposed method to real-world data and further relaxing the assumption of Gaussian noise.\looseness=-1

\section*{Acknowledgments}
The authors would like to thank the reviewers, Yewei Xia, and Yiqing Li for their helpful comments. We would like to acknowledge the support from NSF Award No. 2229881, AI Institute for Societal Decision Making (AI-SDM), the National Institutes of Health (NIH) under Contract R01HL159805, and grants from Quris AI, Florin Court Capital, and MBZUAI-WIS Joint Program. IN acknowledges the support of the Natural Sciences and Engineering Research Council of Canada (NSERC) Postgraduate Scholarships – Doctoral program.

\bibliographystyle{abbrvnat}
\bibliography{ms}

\vspace{-0.2em}
\section*{Checklist}



 \begin{enumerate}

 \item For all models and algorithms presented, check if you include:
 \begin{enumerate}
   \item A clear description of the mathematical setting, assumptions, algorithm, and/or model. [Yes] The assumptions are clearly stated in each theorem statement.
   \item An analysis of the properties and complexity (time, space, sample size) of any algorithm. [Not Applicable] The primary focus of this work is on identifiability theory.
   \item (Optional) Anonymized source code, with specification of all dependencies, including external libraries.
 \end{enumerate}

 \item For any theoretical claim, check if you include:
 \begin{enumerate}
   \item Statements of the full set of assumptions of all theoretical results. [Yes] The assumptions are clearly stated in each theorem statement.
   \item Complete proofs of all theoretical results. [Yes] The proofs are provided in the supplementary materials.
   \item Clear explanations of any assumptions. [Yes]
 \end{enumerate}

 \item For all figures and tables that present empirical results, check if you include:
 \begin{enumerate}
   \item The code, data, and instructions needed to reproduce the main experimental results (either in the supplemental material or as a URL). [Yes]
   \item All the training details (e.g., data splits, hyperparameters, how they were chosen). [Yes] The details are explained in Section 6.
         \item A clear definition of the specific measure or statistics and error bars (e.g., with respect to the random seed after running experiments multiple times). [Not Applicable]
         \item A description of the computing infrastructure used. (e.g., type of GPUs, internal cluster, or cloud provider). [Not Applicable]
 \end{enumerate}

 \item If you are using existing assets (e.g., code, data, models) or curating/releasing new assets, check if you include:
 \begin{enumerate}
   \item Citations of the creator If your work uses existing assets. [Not Applicable]
   \item The license information of the assets, if applicable. [Not Applicable]
   \item New assets either in the supplemental material or as a URL, if applicable. [Not Applicable]
   \item Information about consent from data providers/curators. [Not Applicable]
   \item Discussion of sensible content if applicable, e.g., personally identifiable information or offensive content. [Not Applicable]
 \end{enumerate}

 \item If you used crowdsourcing or conducted research with human subjects, check if you include:
 \begin{enumerate}
   \item The full text of instructions given to participants and screenshots. [Not Applicable]
   \item Descriptions of potential participant risks, with links to Institutional Review Board (IRB) approvals if applicable. [Not Applicable]
   \item The estimated hourly wage paid to participants and the total amount spent on participant compensation. [Not Applicable]
 \end{enumerate}

 \end{enumerate}


\clearpage
\appendix

\onecolumn 

\part{Appendices} 

\section{Extended Discussion of Related Works}\label{app:related_works}
\paragraph{Nonlinear ICA.}
Despite the lack of dependence among latent variables, nonlinear ICA is a challenging problem because the latent variables are generally not identifiable without any further assumptions \citep{hyvarinen1999nonlinear}.
Existing work in nonlinear ICA  often relies on assumptions about changing distributions
where different
distributions are indicated by auxiliary variables such as time indices and domain indices 
\citep{hyvarinen2016unsupervised, hyvarinen2017nonlinear,hyvarinen2019nonlinear, khemakhem2020variational,sorrenson2020disentanglement, lachapelle2024nonparametric,halva2020hidden, halva2021disentangling,yao2022temporally} or specific constraints on the mixing function, including function classes \citep{hyvarinen1999nonlinear, taleb1999source, gresele2021independent, buchholz2022function} or sparsity constraint \citep{zhengidentifiability, zheng2023generalizing}.

\paragraph{Causal representation learning.}
CRL aims to recover the latent variables and their causal relations from data. Without any further assumptions, the identifiability of the hidden generating process is known to be impossible.
One line of work focuses on purely observational data by adding parametric and graphical assumptions. In the linear case, \citet{silva2006learning} recover the Markov equivalence class for the one-factor model, while  \citet{xie2020generalized, cai2019triad} estimate the latent variables and their relations by assuming non-Gaussian noises and at least twice measured variables as pure children as latent ones, which is later extended to latent hierarchical structure \citep{xie2022identification}. Moreover, \citet{adams2021identification} provides necessary and sufficient graphical conditions for the identification of linear non-Gaussian model with latent variables. As for the linear  Gaussian case, \citet{huang2022latent} leverage rank deficiency constraints of sub-covariance matrix over observed variables to identify latent hierarchical structure, which is later extended by  \citet{dong2023versatile} to structures where all variables including latent and observed ones are allowed to be flexibly related. In the discrete setting, \citet{kivva2021learning} also provide identification results for latent causal graphs up to Markov equivalence, by assuming a mixture model with specific graphical assumptions. 

Due to the difficulty of using purely observational data for CRL, researchers have also been working on another line of thought where data from multiple environments are available and the changes of distributions are often assumed to arise from soft or hard interventions. Specifically, \citet{squires2023linear} assume linear SEM and mixing function and leverage single-node interventions for the identifiability,  \citet{varici2023score} consider nonlinear SEM and linear mixing function, \citet{varici2023score, buchholz2023learning, jiang2023learning} concern the setting of the nonparametric latent SEM and linear mixing function, which is further generalized by  \citet{ahuja2023interventional} to nonparametric SEM and polynomial mixing functions, and \citet{brehmer2022weakly, vonkugelgen2023nonparametric, jiang2023learning} consider the general nonparametric settings for both the causal model and mixing function. From the perspective of assumptions about interventions,
\citet{squires2023linear} make use of single-node interventions,  \citet{brehmer2022weakly} rely on counterfactual pre- and post-intervention views, \citet{vonkugelgen2023nonparametric} utilize paired interventions per node, and  \citet{zhang2023identifiability} explore soft interventions. Other works include \citet{yang2021causalvae,shen2022weakly,liang2023causal} that require more supervision information, \citet{kori2023causal} that require a causal ordering prior, \citet{yao2024multiview,xu2024sparsity} that focus on multi-view data, \citet{morioka2024causal} that leverage certain graphical constraint,  \citet{ahuja2023interventional,wang2021desiderata} that utilize further constraint on the latent support, and \citet{lippe2022citris} that provides results for interventions on known targets.

At the same time, many of the existing results rely on assumptions about interventions that may be overly restrictive in practice. As such,  we formalize a set of desiderata about realistic assumptions about changing environments, which we refer to as  CRL from general environments. Under this setting, we propose the first identifiability result that allows nonlinear SEM and nonparametric mixing, while the two strongest existing works either have to require linear mixing \citep{varici2024linear}, or can only identify the latent DAG up to moral graph \citep{zhang2024causal}.

\begin{table*}[t]
    \label{tab:related_works_complete}
    \centering
    \caption{Comparison of of several existing identifiability results of multi-environment CRL based on hard or soft interventions. We only outline the key assumptions and results, while omitting some additional assumptions required by several studies. `Env.' stands for environments.}
    \vskip 0.1in
    \scalebox{0.84}{
    \begin{tabular}{ccccccc}
        \toprule
        \multirow{2}{*}{Work}  & Latent & Mixing & Desiderata & General & Identifiability of  & Identifiability of\\
        & SEM & Function & Satisfied? & Env.? & Latent Variables  & Latent DAG\\
        \midrule
        \citet{liu2023identifying}  & General & Linear & 3  & No & Full  & Full\\
        \hline
        \multirow{2}{*}{\citet{squires2023linear}}  & Linear & Linear & 2 & No & Full  & Full\\
        & Linear & Linear & 1,2 & No & Up to ancestors  & Transitive closure\\
        \hline
        \multirow{2}{*}{\citet{buchholz2023learning}}  & Linear & General & 2  & No & Full  & Full\\
        & Linear & General & 1,2  & No &   & Partial order\\
        \hline
        \citet{zhang2023identifiability}  & Nonlinear & Polynomial & 1,2  & No & Up to ancestors  & Transitive closure\\
        \hline
        \citet{liu2024identifiable}  & Polynomial & General & 3 & No & Full  & Full\\
        \hline
        \multirow{4}{*}{\citet{varici2024scorebased}} & General & Linear & 2 & No & Full  & Full\\
          & General & Linear & 1,2  & No & Up to ancestors  & Transitive closure\\
          & Nonlinear & Linear & 1,2 & No & Up to surrounding  & Full\\
          & General & \textbf{General} & 2 & No & Full  & Full\\
        \hline
        \multirow{2}{*}{\citet{ahuja2023interventional}} & General & Polynomial & 2  & No & Full & Full\\
        & Bounded RV & Polynomial & 1,2  & No & Full & Full\\
        \hline
        \citet{vonkugelgen2023nonparametric}  & General & \textbf{General} & None  & No & Full  & Full\\
        \hline
        \multirow{2}{*}{\citet{jin2023learning}}  & General & \textbf{General} & 1  & No & Up to surrounding & Full\\
          & Linear & Linear & 1,2,3  & \textbf{Yes} & Up to surrounding  & Full\\
        \hline
        \citet{bing2024identifying}  & General & Linear & 2,3  & No & Full  & Full\\
        \hline
        \multirow{2}{*}{\citet{varici2024linear}}  & ANM & Linear & 2,3  & No & Full & Full\\
        & General & Linear & 1,2,3  & \textbf{Yes} & Up to ancestors & Transitive closure\\
        \hline
        \citet{zhang2024causal}  & General & \textbf{General} & 1,2,3 & \textbf{Yes} & Up to intimate neighbors & Moral graph\\
        \hline
        \multirow{2}{*}{\textbf{Ours}}  & ANM & \textbf{General} & 1,2,3 & \textbf{Yes} & Up to surrounding parents & Full\\
          & HNM & \textbf{General} & 1,2,3 & \textbf{Yes} & Up to surrounding parents & Full\\
        \bottomrule
    \end{tabular}
    }
\end{table*}

\paragraph{Ordering-based causal discovery.}
The idea that a causal DAG can be partially represented by a topological ordering has a long history \citep{verma1990causal}. Such a line of thought typically searches the ordering space instead of the DAG space, e.g., with greedy Markov Chain Monte Carlo \citep{friedman2003being}, greedy hill-climbing  \citep{teyssier2012ordering}, restricted maximum likelihood estimators \citep{teyssier2012ordering}, and sparsest permutation \citep{raskutti2018learning,lam2022greedy,solus2021consistency}. Further, assumptions about  functional causal models and noise variances can also be made \citep{ghoshal2018learning,chen2019causal} to sequentially identify leave nodes based on the estimated precision matrix. More recently, score-matching based methods for causal discovery have been proposed. This line of work typically assumes nonlinear additive Gaussian noise models \citep{rolland2022score}, while other settings are also considered \citep{montagna2023assumption,montagna2023scalable,montagna2023causal,sanchez2023diffusion}.

\section{Proof of \cref{proposition:coupled_single_node_interventions}}\label{app:proof_proposition_coupled_single_node_interventions}
\PropositionCoupledSingleNodeInterventions*
\begin{proof}
The ``$\Leftarrow$'' side is trivial. We now prove the ``$\Rightarrow$'' side. For all single-node interventions $I^{(1)},\dots,I^{(m)}$, suppose that we know which of them share the same targets. That is, we can perform a partition of $\{I^{(1)},\dots,I^{(m)}\}$ into $\{\mathcal{I}_1,\dots,\mathcal{I}_n\}$, such that (1) the interventions in $\mathcal{I}_i$ share the same targets, and (2) the interventions in $\mathcal{I}_i$ do not share the same targets with the interventions in $\mathcal{I}_j,j\neq i$. Now suppose we perform a random permutation $\pi$ of $Z$, denoted as $Z_\pi$, and assign $Z_{\pi(i)}$ as the target of interventions in $\mathcal{I}_i$. Clearly, one of the permutation $\pi^*$ will lead to the correct assignment of intervention targets, but is unknown to us. This implies that we know the intervention targets (e.g., the assignment we perform with $Z_\pi$) up to variable permutation.
\end{proof}

\section{Proofs of \cref{lemma:zero_derivative_anm,lemma:zero_derivative_hnm}}\label{app:proof_lemmas_zero_derivative}
\subsection{Proof of \cref{lemma:zero_derivative_anm}}\label{app:proof_lemmas_zero_derivative_anm}
This lemma is implied by \citet{rolland2022score}, and we provide the proof here for completeness.
\LemmaZeroDerivativeANM*
\begin{proof}
Since the distribution $P_Z^{(u)}$ and the DAG $\mathcal{G}_Z$ satisfy the Markov property, we have
\[
p^{(u)}(Z)=\prod_{k=1}^n p^{(u)}(Z_k\,|\,\textrm{PA}(Z_k; \mathcal{G}_Z)).
\]
By assumption, $Z_i$ is a sink node in DAG $\mathcal{G}_Z$, which implies
\begin{flalign*}
\frac{\partial^2 \log p^{(u)}(Z)}{\partial Z_i^2}&=\frac{\partial^2 \log p^{(u)}(Z_i\,|\,\textrm{PA}(Z_i; \mathcal{G}_Z))}{\partial Z_i^2}\\
&=\frac{\partial^2}{\partial Z_i^2}\left(-\frac{1}{2}\left(\frac{Z_i - f_i^{(u)}(\textrm{PA}(Z_i;\mathcal{G}_Z))}{\sigma_i^{(u)}}\right)^2-\frac{1}{2}\log (2\pi (\sigma_i^{(u)})^2)\right)\\
&=-\frac{1}{\sigma_i^{(u)}}\cdot\frac{\partial}{\partial Z_i}\left(\frac{Z_i - f_i^{(u)}(\textrm{PA}(Z_i;\mathcal{G}_Z))}{\sigma_i^{(u)}}\right)\\
&=-\frac{1}{(\sigma_i^{(u)})^2},
\end{flalign*}
where $(\sigma_i^{(u)})^2$ is the variance of $\epsilon_i^{(u)}$. Therefore, for any $Z_j$, we have
\[
\frac{\partial^3 \log p^{(u)}(Z)}{\partial Z_i^2 \partial Z_j}=0.
\]
\end{proof}

\subsection{Proof of \cref{lemma:zero_derivative_hnm}}\label{app:proof_lemmas_zero_derivative_hnm}
\LemmaZeroDerivativeHNM*
\begin{proof}
Since the distribution $P_Z^{(u)}$ and the DAG $\mathcal{G}_Z$ satisfy the Markov property, we have
\[
p^{(u)}(Z)=\prod_{k=1}^n p^{(u)}(Z_k\,|\,\textrm{PA}(Z_k; \mathcal{G}_Z)).
\]
By assumption, $Z_i$ is a sink node in DAG $\mathcal{G}_Z$, which implies
\begin{flalign*}
\frac{\partial^2 \log p^{(u)}(Z)}{\partial Z_i^2}&=\frac{\partial^2 \log p^{(u)}(Z_i\,|\,\textrm{PA}(Z_i; \mathcal{G}_Z))}{\partial Z_i^2}\\
&=\frac{\partial^2}{\partial Z_i^2}\left(-\frac{1}{2}\left(\frac{Z_i - f_i^{(u)}(\textrm{PA}(Z_i;\mathcal{G}_Z))}{\sigma_i^{(u)}(\textrm{PA}(Z_i;\mathcal{G}_Z))}\right)^2\right)\\
&=-\frac{1}{\sigma_i^{(u)}(\textrm{PA}(Z_i;\mathcal{G}_Z))}\cdot\frac{\partial}{\partial Z_i}\left(\frac{Z_i - f_i^{(u)}(\textrm{PA}(Z_i;\mathcal{G}_Z))}{\sigma_i^{(u)}(\textrm{PA}(Z_i;\mathcal{G}_Z))}\right)\\
&=-\frac{1}{(\sigma_i^{(u)}(\textrm{PA}(Z_i;\mathcal{G}_Z)))^2}.
\end{flalign*}
Therefore, for any $Z_j\not\in \textrm{PA}(Z_i;\mathcal{G}_Z)$, we have
\[
\frac{\partial^3 \log p^{(u)}(Z)}{\partial Z_i^2 \partial Z_j}=0.
\]
\end{proof}

\vspace{-1em}
\section{Proof of \cref{theorem:identifiability_anm}}
\label{app:theorem_identifiability_anm}
\ThmIdentifiabilityANM*
\begin{proof}
Let $(\hat{g},p_{\hat{Z}},\mathcal{G}_{\hat{Z}})$ be an output of Line 3 in \cref{alg:crl} during the $(n-1)$-th iteration, which corresponds to the output of \cref{alg:crl}. Since both the true mixing function $g$ and the estimated mixing function $\hat{g}$ are diffeomorphisms onto their images, the transformation from $\hat{Z}$ to $Z$ is also a diffeomorphism. Therefore, there exists a permutation  such that the Jacobian matrix $\frac{\partial Z_{\alpha}}{\partial \hat{Z}}$ has nonzero diagonal entries (e.g., see \citet[Lemma~2]{zhang2024causal}). Denote $Z_{\alpha[k]}\coloneqq(Z_{\alpha(l)})_{l=1}^{k}$. By the faithfulness assumption and \cref{proposition:induction_anm_dag} for the case of $(n-1)$-th iteration, we conclude that the DAGs $\mathcal{G}_{\hat{Z}}$ and $\mathcal{G}_{Z_\alpha}$ are identical, and that each latent variable $\hat{Z}_{i},i\in[n]$ is solely a function of a subset of $Z_{\alpha(i)}\cup \bigcap_{k=i}^{n}\Psi(Z_{\alpha(i)};\mathcal{M}_{Z_{\alpha[k]}})$.

By definition, $\hat{Z}$ follows a causal order with respect to the latent DAG $\mathcal{G}_{\hat{Z}}$. Since $\mathcal{G}_{\hat{Z}}$ and $\mathcal{G}_{Z_\alpha}$ are identical, $\alpha$ is also a causal order of $Z$ with respect to $\mathcal{G}_{Z}$. Under the faithfulness assumption, \cref{proposition:intimate_parents} implies that $\hat{Z}_{i}$ is solely a function of a subset of
\[
Z_{\alpha(i)}\cup \bigcap_{k=i}^{n}\Psi(Z_{\alpha(i)};\mathcal{M}_{Z_{\alpha[k]}})\subseteq Z_{\alpha(i)}\cup \operatorname{sur}(Z_{\alpha(i)};\mathcal{G}_{Z}).
\]
By taking $\pi\coloneqq\alpha^{-1}$, we obtain the desired statements.
\end{proof}

\subsection{Second-Order Partial Derivative of Latent Distribution}
We provide the following lemma which will be used as the starting point to derive the third-order derivative in \cref{eq:third_order_derivative} in the proof of \cref{proposition:induction_anm_dag}. This lemma is obtained from \citet[Proposition~1]{zhang2024causal}. The proof involves change-of-variable formula \citep{benisrael1999change,gemici2016normalizing}, chain rule, and property of Markov network~\citep{lin1997factorizing}, which is omitted here.
\begin{lemma}[Second-order derivative]\label{lemma:second_order_derivative}
Consider the data generating process in \cref{eq:data_generating_process}. Suppose that we learn $(\hat{g},p_{\hat{Z}},\mathcal{G}_{\hat{Z}})$ to achieve \cref{eq:observational_equivalence}. Then, we have
\begin{flalign*}
\frac{\partial^2\log p^{(u)}(\hat{Z})}{\partial \hat{Z}_k \partial \hat{Z}_l} =& \sum_{i=1}^n \frac{\partial^2 \log p^{(u)}({Z})}{\partial Z_i^2} \frac{\partial  Z_i}{\partial \hat{Z}_l} \frac{\partial Z_i}{\partial \hat{Z}_k} +  \sum_{j=1}^n \sum_{\substack{i:\{Z_j,Z_i\}\in \mathcal{E}(\mathcal{M}_Z)}} \frac{\partial^2 \log p^{(u)}({Z})}{\partial Z_i \partial Z_j} \frac{\partial  Z_j}{\partial \hat{Z}_l} \frac{\partial Z_i}{\partial \hat{Z}_k} \\
& \qquad +\sum_{i=1}^n \frac{\partial \log p^{(u)}({Z})}{\partial  Z_i}\frac{\partial^2 Z_i}{\partial \hat{Z}_k \partial \hat{Z}_l} + \frac{\partial^2\log|\det J_v|}{\partial \hat{Z}_k \partial \hat{Z}_l},
\end{flalign*}
where $v\coloneqq  g^{-1}\circ\hat{g}$ is a diffeomorphism.
\end{lemma}

\subsection{Relation between Markov Network and Moral Graph}
We provide the following lemma to relate the Markov network and moral graph, which is useful for proving \cref{proposition:induction_anm_dag,proposition:intimate_parents}.
\begin{lemma}[{\citet[Lemma~1]{zhang2024causal}}]\label{lemma:moral_graph_subgraph}
Consider a DAG $\mathcal{G}_Z$ and distribution $P_{Z}$ with its Markov Network $\mathcal{M}_Z$. Suppose that the Markov assumption holds. Then, the undirected graph defined by $\mathcal{M}_Z$ is a subgraph of the moral graph of DAG $\mathcal{G}_Z$.
\end{lemma}

\subsection{Proof of \cref{lemma:sub_diffeomorphism}}
The following lemma shows how diffeomorphism exists between subset of variables, which is useful for proving \cref{proposition:induction_anm_dag}.
\begin{lemma}\label{lemma:sub_diffeomorphism}
Suppose that the transformation from $\hat{Z}\in\mathbb{R}^n$ to $Z\in\mathbb{R}^n$ is a diffeomorphism, and that there exists $\mathcal{I}\subseteq[n]$ such that each $Z_i,i\in \mathcal{I}$ is not a function of $\hat{Z}_j,j\not\in \mathcal{I}$. Then, the transformation from $\hat{Z}_{ \mathcal{I}}$ to $Z_{\mathcal{I}}$ is a diffeomorphism.
\end{lemma}
\begin{proof}
Let $J\coloneqq\frac{\partial Z}{\partial \hat{Z}}$ be the Jacobian matrix of the transformation from $\hat{Z}$ to $Z$. Since the transformation is a diffeomorphism, the Jacobian matrix $J$ is invertible. By the Cayley-Hamilton theorem, for each value, its inverse $\frac{\partial \hat{Z}}{\partial Z}=J^{-1}$ can be written as $\sum_{k=0}^{n} c_k J^k$ for some coefficients $c_0,\dots,c_{n}$. For $i\in \mathcal{I}$ and $j\not\in \mathcal{I}$, since $Z_i$ is, by definition, not a function of $\hat{Z}_j$, we clearly have $J_{i,j}=0$. Also, it is straightforward to show $J_{i,:}J_{:,j}=0$, which indicates $(J^2)_{i,j}=0$. By mathematical induction, we obtain $(J^k)_{i,j}=0$. This then implies
\[\frac{\partial \hat{Z}_i}{\partial Z_j}=\left(\frac{\partial \hat{Z}}{\partial Z}\right)_{i,j}=(J^{-1})_{i,j}=\left(\sum_{k=0}^{n} c_k J^k\right)_{i,j}=0.\]
That is, each $\hat{Z}_i,i\in \mathcal{I}$ is not a function of $Z_j,j\not\in \mathcal{I}$. This implies that each $Z_i,i\in \mathcal{I}$ is solely a function of $\hat{Z}_{\mathcal{I}}$, and vice versa, i.e., each $\hat{Z}_i,i\in \mathcal{I}$ is solely a function of $Z_{\mathcal{I}}$. Therefore, we conclude that the transformation from $\hat{Z}_{ \mathcal{I}}$ to $Z_{\mathcal{I}}$ is a diffeomorphism.
\end{proof}

\subsection{Proof of \cref{proposition:induction_anm_dag}}
We prove the following result via mathematical induction, which is crucial for the proof of \cref{theorem:identifiability_anm}. Specifically, it provides information on the identifiability of the latent DAG $\mathcal{G}_Z$ and latent variables $Z$ in each iteration of \cref{alg:crl}.

\begin{proposition}[Output of $t$-th iteration]\label{proposition:induction_anm_dag}
Consider the data generating process in \cref{eq:data_generating_process,eq:anm}. Suppose that \cref{assumption:sufficient_changes_markov_network,assumption:sufficient_changes_intimate_neighbors_anm}, as well as the faithfulness assumption, hold. Let $(\hat{g},p_{\hat{Z}},\mathcal{G}_{\hat{Z}})$ be an output of Line 3 in \cref{alg:crl} during the $t$-th iteration. Let $\alpha$ be a permutation such that the Jacobian matrix $\frac{\partial Z_{\alpha}}{\partial \hat{Z}}$ has nonzero diagonal entries. Then, we have the following statements:
\begin{enumerate}[label=(\alph*)]
\item For $i\in [n],j=n+1-t,\dots,n$, we have $\hat{Z}_{i}\rightarrow\hat{Z}_{j}$ in $\mathcal{G}_{\hat{Z}}$ if and only if $Z_{\alpha(i)}\rightarrow Z_{\alpha(j)}$ in $\mathcal{G}_Z$.
\item Each latent variable $\hat{Z}_{i},i\in[n]$ is solely a function of a subset of $Z_{\alpha(i)}\cup \bigcap_{k=\max(n+1-t,i)}^{n}\Psi(Z_{\alpha(i)};\mathcal{M}_{Z_{\alpha[k]}})$.
\item The transformation from $\hat{Z}_{[n-t]}\coloneqq(\hat{Z}_{i})_{i=1}^{n-t}$ to $Z_{\alpha[n-t]}\coloneqq(Z_{\alpha(i)})_{i=1}^{n-t}$ is a diffeomorphism, and the Jacobian matrix $\frac{\partial Z_{\alpha[n-t]}}{\partial \hat{Z}_{[n-t]}}$ has nonzero diagonal entries.
\end{enumerate}
\end{proposition}
\begin{proof}
We prove the proposition by mathematical induction from $t=1$ to $n-1$. We first provide the proof for the inductive step. Suppose that \cref{proposition:induction_anm_dag} holds for the case of $t$. In the following, we show that it also holds for the case of $t+1$.

By definition, $(\hat{g},p_{\hat{Z}},\mathcal{G}_{\hat{Z}})$ denotes the output of Line 3 in \cref{alg:crl} during the $(t+1)$-th iteration, and $\alpha$ denotes a permutation such that the Jacobian matrix $\frac{\partial Z_{\alpha}}{\partial \hat{Z}}$ has nonzero diagonal entries, where $\hat{Z}$ is the estimated latent variables from Line 3 in \cref{alg:crl} during the $(t+1)$-th iteration. Due to the constraints (i) and (ii) in Line 2 of \cref{alg:crl}, $(\hat{g},p_{\hat{Z}},\mathcal{G}_{\hat{Z}})$ is also a valid output of Line 3 in \cref{alg:crl} during the $t$-th iteration. (In this case, the same permutation $\alpha$ can also be used in the $t$-th iteration such that the corresponding Jacobian matrix$\frac{\partial Z_{\alpha}}{\partial \hat{Z}}$  has nonzero diagonal entries.) Therefore, by the induction hypothesis, Statements (a), (b), and (c) of \cref{proposition:induction_anm_dag} hold for $(\hat{g},p_{\hat{Z}},\mathcal{G}_{\hat{Z}})$ for the case of $t$, and we aim to further show that these statements also hold for $(\hat{g},p_{\hat{Z}},\mathcal{G}_{\hat{Z}})$ for the case of $t+1$.

By Statement (c) of \cref{proposition:induction_anm_dag} for the case of $t$, the transformation from $\hat{Z}_{[n-t]}$ to $Z_{\alpha[n-t]}$, denoted by $v$, is a diffeomorphism. By the change-of-variable formula, we have
\begin{flalign}
p^{(u)}(\hat{Z}_{[n-t]})|\det J_{v^{-1}}| &= p^{(u)}(Z_{\alpha[n-t]})\nonumber \nonumber\\
\log p^{(u)}(\hat{Z}_{[n-t]}) &= \log p^{(u)}(Z_{\alpha[n-t]}) + \log|\det J_{v}|,\label{eq:change_of_variable}
\end{flalign}
where $J_{v}$ denotes the Jacobian matrix of $v$. Suppose $\hat{Z}_{l}$ is a sink node in DAG $\mathcal{G}_{\hat{Z}_{[n-t]}}$. By \cref{lemma:second_order_derivative}, the second-order derivative w.r.t $\hat{Z}_{l}$ is given by
\begin{flalign*}
\hspace{-0.5em}\frac{\partial^2 \log p^{(u)}(\hat{Z})}{\partial \hat{Z}_{l}^2}=& \sum_{i=t}^n \frac{\partial^2 \log p^{(u)}(Z_{\alpha[n-t]})}{\partial Z_{\alpha(i)}^2} \left(\frac{\partial  Z_{\alpha(i)}}{\partial \hat{Z}_{l}}\right)^2 +  \sum_{j=t}^n \sum_{\substack{i:\{Z_{\alpha(i)},Z_{\alpha(j)}\}\in \mathcal{E}(\mathcal{M}_{Z_{\alpha[n-t]}})}}\frac{\partial^2 \log p^{(u)}(Z_{\alpha[n-t]})}{\partial Z_{\alpha(i)} \partial Z_{\alpha(j)}} \frac{\partial Z_{\alpha(j)}}{\partial \hat{Z}_{l}} \frac{\partial Z_{\alpha(i)}}{\partial \hat{Z}_{l}} \\
&=\sum_{i=t}^n \frac{\partial \log p^{(u)}(Z_{\alpha[n-t]})}{\partial  Z_{\alpha(i)}}\frac{\partial^2 Z_{\alpha(i)}}{\partial \hat{Z}_{l}^2} + \frac{\partial^2\log|\det J_{v}|}{\partial \hat{Z}_l^2}.
\end{flalign*}
Further taking third-order derivative w.r.t the sink node $\hat{Z}_l$ and applying \cref{lemma:zero_derivative_anm}, we have
\begin{equation}\label{eq:third_order_derivative}
\begin{aligned}
0 = &\sum_{i:Z_{\alpha(i)}\in Z_{\alpha[n-t]}\setminus \mathcal{S}(\mathcal{G}_{Z_{\alpha[n-t]}})} \frac{\partial^3 \log p^{(u)}(Z_{\alpha[n-t]})}{\partial Z_{\alpha(i)}^3} \left(\frac{\partial  Z_{\alpha(i)}}{\partial \hat{Z}_{l}}\right)^3 \\
&\qquad+\sum_{\substack{i,j: \\\{Z_{\alpha(i)},Z_{\alpha(j)}\}\in \mathcal{E}(\mathcal{M}_{Z_{\alpha[n-t]}}),\\ Z_{\alpha(i)}\not\in\mathcal{S}(\mathcal{G}_{Z_{\alpha[n-t]}})}} \frac{\partial^3 \log p^{(u)}(Z_{\alpha[n-t]})}{\partial Z_{\alpha(i)}^2 \partial Z_{\alpha(j)}} \left(2\cdot\left(\frac{\partial  Z_{\alpha(i)}}{\partial \hat{Z}_{l}}\right)^2 \frac{\partial Z_{\alpha(j)}}{\partial \hat{Z}_{l}}\right)\\
&\qquad+\sum_{\substack{i,j,k:\\i<j<k,\\\{Z_{\alpha(i)},Z_{\alpha(j)}\},\{Z_{\alpha(j)},Z_{\alpha(k)}\},\{Z_{\alpha(i)},Z_{\alpha(k)}\}\in\mathcal{E}(\mathcal{M}_{Z_{\alpha[n-t]}})}} \frac{\partial^3 \log p^{(u)}(Z_{\alpha[n-t]})}{\partial Z_{\alpha(i)} \partial Z_{\alpha(j)}\partial Z_{\alpha(k)}} \left(6\cdot\frac{\partial  Z_{\alpha(i)}}{\partial \hat{Z}_{l}} \frac{\partial Z_{\alpha(j)}}{\partial \hat{Z}_{l}}\frac{\partial  Z_{\alpha(k)}}{\partial \hat{Z}_{l}}\right) \\
& \qquad+ \sum_{i=1}^d \frac{\partial^2 \log p^{(u)}(Z_{\alpha[n-t]})}{\partial Z_{\alpha(i)}^2} \left(3\cdot\frac{\partial^2  Z_{\alpha(i)}}{\partial \hat{Z}_{l}^2} \frac{\partial Z_{\alpha(i)}}{\partial \hat{Z}_{l}}\right)\\
& \qquad+ \sum_{\substack{i,j: \\i < j,\\ \{Z_{\alpha(i)},Z_{\alpha(j)}\}\in \mathcal{E}(\mathcal{M}_{Z_{\alpha[n-t]}})}} \frac{\partial^2 \log p^{(u)}(Z_{\alpha[n-t]})}{\partial Z_{\alpha(i)} \partial Z_{\alpha(j)}} \left(3\cdot\frac{\partial^2  Z_{\alpha(j)}}{\partial \hat{Z}_{l}^2} \frac{\partial Z_{\alpha(i)}}{\partial \hat{Z}_{l}}+3\cdot\frac{\partial  Z_{\alpha(j)}}{\partial \hat{Z}_{l}} \frac{\partial^2 Z_{\alpha(i)}}{\partial \hat{Z}_{l}^2}\right)\\
&\qquad + 
\sum_{i=1}^n \frac{\partial \log p^{(u)}(Z_{\alpha[n-t]})}{\partial  Z_{\alpha(i)}}\frac{\partial^3 Z_{\alpha(i)}}{\partial \hat{Z}_{l}^3} + \frac{\partial^3\log|\det J_{v}|}{\partial \hat{Z}_{l}^3}.
\end{aligned}
\end{equation}
By \cref{assumption:sufficient_changes_intimate_neighbors_anm}, there exist multiple values of $u_j$ such that the above equation holds. Subtracting each equation corresponding to $u_j,j\neq 0$ with the equation corresponding to $u_0$, and using the assumption that the vectors formed by collecting the resulting coefficients (involving differences of partial derivatives) are linearly independent, we obtain
\begin{equation}\label{eq:proof_third_derivative_zero}
\left(\frac{\partial  Z_{\alpha(i)}}{\partial \hat{Z}_{l}}\right)^3=0 \quad\iff \quad\frac{\partial  Z_{\alpha(i)}}{\partial \hat{Z}_{l}}=0\quad \text{for} \quad Z_{\alpha(i)}\in Z_{\alpha[n-t]}\setminus \mathcal{S}(\mathcal{G}_{Z_{\alpha[n-t]}}).
\end{equation}
Therefore, the non-sink nodes in $\mathcal{G}_{Z_{\alpha[n-t]}}$ cannot be a function of $\hat{Z}_{l}$ which is a sink node in in DAG $\mathcal{G}_{\hat{Z}_{[n-t]}}$. This implies that $Z_{\alpha(l)}$ is a sink node in DAG $\mathcal{G}_{Z_{\alpha[n-t]}}$, because otherwise $Z_{\alpha(l)}$ is not a function of $\hat{Z}_{l}$, which contradicts the induction hypothesis that the Jacobian matrix $\frac{\partial Z_{\alpha[n-t]}}{\partial \hat{Z}_{[n-t]}}$ has nonzero diagonal entries. 

Recall that in Line 2 of \cref{alg:crl}, the estimated Markov network $\mathcal{M}_{\hat{Z}_{[n-t]}}$ over $\hat{Z}_{[n-t]}$ has minimal number of edges. Thus, by \cref{eq:change_of_variable}, \cref{assumption:sufficient_changes_markov_network}, and the induction hypothesis that the Jacobian matrix $\frac{\partial Z_{\alpha[n-t]}}{\partial \hat{Z}_{[n-t]}}$ has nonzero diagonal entries, \citet[Theorem~2]{zhang2024causal} implies that the Markov newtorks over $\hat{Z}_{[n-t]}$ and $Z_{\alpha[n-t]}$ are identical, i.e., 
\begin{equation}\label{eq:identical_markov_network}
\{\hat{Z}_{i},\hat{Z}_{j}\}\in\mathcal{E}(\mathcal{M}_{\hat{Z}_{[n-t]}}) \quad\iff\quad \{Z_{\alpha(i)}, Z_{\alpha(j)}\}\in\mathcal{E}(\mathcal{M}_{Z_{\alpha[n-t]}}),\quad i,j\in [n-t], i\neq j.
\end{equation}
The faithfulness assumption implies that the undirected graph defined by Markov network $\mathcal{M}_{Z_{\alpha[n-t]}}$ is the moral graph of DAG $\mathcal{G}_{Z_{\alpha[n-t]}}$ (e.g., see \citet[Proposition~2]{zhang2024causal}). Since $Z_{\alpha(l)}$ is a sink node in DAG $\mathcal{G}_{Z_{\alpha[n-t]}}$, each undirected edge $\{Z_{\alpha(i)}, Z_{\alpha(l)}\}\in\mathcal{E}(\mathcal{M}_{Z_{\alpha[n-t]}})$ in the moral graph implies a directed edge $Z_{\alpha(i)}\rightarrow Z_{\alpha(l)}$ in DAG $\mathcal{G}_Z$. Similar reasoning can be used to show that each undirected edge $\{\hat{Z}_{i},\hat{Z}_{l}\}\in\mathcal{E}(\mathcal{M}_{\hat{Z}_{[n-t]}})$ in the moral graph implies a directed edge $\hat{Z}_{i}\rightarrow\hat{Z}_{l}$ in DAG $\mathcal{G}_{\hat{Z}}$. Therefore, we have
\[
\hat{Z}_{i}\rightarrow\hat{Z}_{l} \text{~in~} \mathcal{G}_{\hat{Z}} \quad\iff\quad Z_{\alpha(i)}\rightarrow Z_{\alpha(l)} \text{~in~} \mathcal{G}_Z,\quad i\in [n].
\]
In Line 3 of \cref{alg:crl}, after reordering $\hat{Z}$ in $(\hat{g},p_{\hat{Z}},\mathcal{G}_{\hat{Z}})$, we know that $\hat{Z}_{n-t}$ is a sink node in DAG $\mathcal{G}_{\hat{Z}_{[n-t]}}$. By the above relation, we then have
\begin{equation}\label{eq:induction_new_edges}
\hat{Z}_{i}\rightarrow\hat{Z}_{n-t} \text{~in~} \mathcal{G}_{\hat{Z}} \quad\iff\quad Z_{\alpha(i)}\rightarrow Z_{\alpha(n-t)} \text{~in~} \mathcal{G}_Z,\quad i\in [n].
\end{equation}

Recall that the induction hypothesis (i.e., Statement (a) of \cref{proposition:induction_anm_dag} for the case of $t$) indicates
\begin{equation}\label{eq:induction_existing_edges}
\hat{Z}_{i}\rightarrow\hat{Z}_{j} \text{~in~} \mathcal{G}_{\hat{Z}} \quad\iff\quad Z_{\alpha(i)}\rightarrow Z_{\alpha(j)} \text{~in~} \mathcal{G}_Z,\quad i\in [n],j=n+1-t,\dots,n.
\end{equation}
By \cref{eq:induction_new_edges,eq:induction_existing_edges}, we have shown that Statement (a) of \cref{proposition:induction_anm_dag} holds for the case of $t + 1$.

Now we prove Statements (b) and (c) of \cref{proposition:induction_anm_dag} for the case of $t + 1$. Recall that in Line 2 of \cref{alg:crl}, the estimated Markov network $\mathcal{M}_{\hat{Z}_{[n-t]}}$ over $\hat{Z}_{[n-t]}$ has minimal number of edges. Thus, by \cref{eq:change_of_variable}, \cref{assumption:sufficient_changes_markov_network}, and the induction hypothesis that the Jacobian matrix $\frac{\partial Z_{\alpha[n-t]}}{\partial \hat{Z}_{[n-t]}}$ has nonzero diagonal entries, \citet[Proposition~4]{zhang2024causal} and its proof imply the following statements:
\begin{itemize}
\item Statement(i): Each latent variable $Z_{\alpha(i)},i\in [n-t]$ is solely a function of a subset of $\hat{Z}_{i}\cup \Psi(\hat{Z}_i;\mathcal{M}_{\hat{Z}_{[n-t]}})$.\item Statement(ii):  Each latent variable $\hat{Z}_{i},i\in [n-t]$ is solely a function of a subset of $Z_{\alpha(i)}\cup \Psi(Z_{\alpha(i)};\mathcal{M}_{Z_{\alpha[n-t]}})$. 
\end{itemize}

By the induction hypothesis (i.e., Statement (b) of \cref{proposition:induction_anm_dag} for the case of $t$), each latent variable $\hat{Z}_{i},i\in[n]$ is solely a function of a subset of $Z_{\alpha(i)}\cup \bigcap_{k=\max(n+1-t,i)}^{n}\Psi(Z_{\alpha(i)};\mathcal{M}_{Z_{\alpha[k]}})$. By Statement (ii), one can straightforwardly show that each latent variable $\hat{Z}_{i},i\in[n]$ is solely a function of a subset of $Z_{\alpha(i)}\cup \bigcap_{k=\max(n-t,i)}^{n}\Psi(Z_{\alpha(i)};\mathcal{M}_{Z_{\alpha[k]}})$, indicating that Statement (b) of \cref{proposition:induction_anm_dag} hold for the case of $t + 1$.

In Line 3 of \cref{alg:crl}, after reordering $\hat{Z}$ in $(\hat{g},p_{\hat{Z}},\mathcal{G}_{\hat{Z}})$, we know that $\hat{Z}_{n-t}$ is a sink node in DAG $\mathcal{G}_{\hat{Z}_{[n-t]}}$. First, by \cref{eq:proof_third_derivative_zero}, the non-sink nodes in $\mathcal{G}_{Z_{\alpha[n-t]}}$ cannot be a function of $\hat{Z}_{n-t}$. This also indicates that $Z_{\alpha(n-t)}$ is a sink node in $\mathcal{G}_{Z_{\alpha[n-t]}}$, due to the induction hypothesis that the Jacobian matrix $\frac{\partial Z_{\alpha[n-t]}}{\partial \hat{Z}_{[n-t]}}$ has nonzero diagonal entries. Second, suppose that $Z_{\alpha(k)}$ is a sink node in $\mathcal{G}_{Z_{\alpha[n-t]}}$, and that it is distinct from $Z_{\alpha(n-t)}$. Since both $Z_{\alpha(n-t)}$ and $Z_{\alpha(k)}$ are sink nodes in DAG $\mathcal{G}_{Z_{\alpha[n-t]}}$, clearly they cannot be adjacent to each other or share a common child in DAG $\mathcal{G}_{Z_{\alpha[n-t]}}$. By \cref{lemma:moral_graph_subgraph}, $Z_{\alpha(n-t)}$ and $Z_{\alpha(k)}$ cannot be adjacent in Markov network $\mathcal{M}_{Z_{\alpha[n-t]}}$. By \cref{eq:identical_markov_network}, $\hat{Z}_{n-t}$ and $\hat{Z}_{k}$ cannot be adjacent in Markov network $\mathcal{M}_{\hat{Z}_{[n-t]}}$. This implies $\hat{Z}_{n-t}\not\in\Psi(\hat{Z}_k;\mathcal{M}_{\hat{Z}_{[n-t]}})$, which, by Statement (i) derived above, indicates that $Z_{\alpha(k)}$ cannot be a function of $\hat{Z}_{n-t}$. Combining both cases (for sink nodes and non-sink nodes), we conclude that all variables in $Z_{\alpha[n-t]}$, except $Z_{\alpha(n-t)}$, cannot be a function of $\hat{Z}_{n-t}$. Therefore, by the induction hypothesis about diffeomorphism from $\hat{Z}_{[n-t]}$ to $Z_{\alpha[n-t]}$, applying \cref{lemma:sub_diffeomorphism} implies that the transformation from $\hat{Z}_{[n-1-t]}$ to $Z_{\alpha[n-1-t]}$ is a diffeomorphism. Also, clearly the Jacobian matrix $\frac{\partial Z_{\alpha[n-1-t]}}{\partial \hat{Z}_{[n-1-t]}}$ has nonzero diagonal entries. That is, we have shown that Statement (c) of \cref{proposition:induction_anm_dag} holds for the case of $t + 1$.

Up until now, we have shown that the inductive step holds for the proof of \cref{proposition:induction_anm_dag}. The same technique applies to the base case of $t=1$, which is omitted here. Specifically, for the base case, we rely on the assumption that both the true mixing function $g$ and the estimated mixing function $\hat{g}$ are diffeomorphisms onto their images, indicating that there the transformation from $\hat{Z}$ to $Z$ is a diffeomorphism.
\end{proof}

\vspace{0.1em}
\subsection{Proof of \cref{proposition:intimate_parents}}
The following result is used in the proof of \cref{theorem:identifiability_anm}, specifically for the identifiability of latent variables $Z$. It relates the intimate neighbors in different Markov networks to the surrounding parents in the latent DAG.

\begin{proposition}\label{proposition:intimate_parents}
Consider the data generating process in \cref{eq:data_generating_process}. Let $\alpha$ be a causal order of variables $Z=(Z_1,\dots,Z_n)$ with respect to DAG $\mathcal{G}_{Z}$ and denote $Z_{\alpha[k]}\coloneqq(Z_{\alpha(l)})_{l=1}^{k}$. Under the faithfulness assumption, we have
\[ \bigcap_{k=i}^{n}\Psi(Z_{\alpha(i)};\mathcal{M}_{Z_{\alpha[k]}})\subseteq  \operatorname{sur}(Z_{\alpha(i)};\mathcal{G}_{Z}) \quad\text{for}\quad i\in[n-1].
\]
\end{proposition}
\begin{proof}
It suffices to show that $Z_{\alpha(j)}\in\bigcap_{k=i}^{n}\Psi(Z_{\alpha(i)};\mathcal{M}_{Z_{\alpha[k]}})$ implies $Z_{\alpha(j)}\in \operatorname{sur}(Z_{\alpha(i)};\mathcal{G}_{Z})$ for $i\in[n-1]$. We provide a proof by contrapositive, i.e., we aim to show that $Z_{\alpha(j)}\not\in \operatorname{sur}(Z_{\alpha(i)};\mathcal{G}_{Z})$ implies
\begin{equation}\label{eq:proof_not_in_intersection}
Z_{\alpha(j)}\not\in\bigcap_{k=i}^{n}\Psi(Z_{\alpha(i)};\mathcal{M}_{Z_{\alpha[k]}}).
\end{equation}
Now suppose $Z_{\alpha(j)}\not\in \operatorname{sur}(Z_{\alpha(i)};\mathcal{G}_{Z})$. By the definition of surrounding parents, we have the following cases.

\textbf{Case 1:} $Z_{\alpha(j)}$ is not a parent of $Z_{\alpha(i)}$. We need to consider the following cases.
\begin{itemize}
\item \textbf{Case 1(a):} $Z_{\alpha(j)}$ is a child of $Z_{\alpha(i)}$. This means $i<j$. For $i\leq k<j$, we have $Z_{\alpha(j)}\not\in Z_{\alpha[k]}$ and thus $Z_{\alpha(j)}\not\in \Psi(Z_{\alpha(i)};\mathcal{M}_{Z_{\alpha[k]}})$, which implies \cref{eq:proof_not_in_intersection}.
\item \textbf{Case 1(b):} $Z_{\alpha(j)}$ is a spouse of $Z_{\alpha(i)}$. By definition of spouse, $Z_{\alpha(j)}$ is not adjacent to $Z_{\alpha(i)}$ in $\mathcal{G}_Z$. If $i<j$, then for $i\leq k<j$, we have $Z_{\alpha(j)}\not\in Z_{\alpha[k]}$ and thus $Z_{\alpha(j)}\not\in \Psi(Z_{\alpha(i)};\mathcal{M}_{Z_{\alpha[k]}})$, which implies \cref{eq:proof_not_in_intersection}. Otherwise, we have $i>j$. In this case, consider the latent DAG $\mathcal{G}_{Z_{\alpha[i]}}$, where $Z_{\alpha(i)}$ is a sink node. Clearly, $Z_{\alpha(j)}$ is not a parent, child, or spouse of $Z_{\alpha(i)}$ in DAG $\mathcal{G}_{Z_{\alpha[i]}}$. Therefore, $Z_{\alpha(j)}$ and $Z_{\alpha(i)}$ are not adjacent in the moral graph of $\mathcal{G}_{Z_{\alpha[i]}}$, which, by \cref{lemma:moral_graph_subgraph}, indicates that they are not adjacent in the Markov network $\mathcal{M}_{Z_{\alpha[i]}}$. This implies $Z_{\alpha(j)}\not\in\Psi(Z_{\alpha(i)};\mathcal{M}_{Z_{\alpha[i]}})$ and thus \cref{eq:proof_not_in_intersection}.
\item \textbf{Case 1(c):} $Z_{\alpha(j)}$ is not a child or spouse of $Z_{\alpha(i)}$. In this case, $Z_{\alpha(j)}$ is not a parent, child, or spouse of $Z_{\alpha(i)}$. Therefore, $Z_{\alpha(j)}$ and $Z_{\alpha(i)}$ are not adjacent in the moral graph of $\mathcal{G}_Z$, which, by \cref{lemma:moral_graph_subgraph}, indicates that they are not adjacent in the Markov network $\mathcal{M}_Z$. This implies $Z_{\alpha(j)}\not\in\Psi(Z_{\alpha(i)};\mathcal{M}_{Z})$ and thus \cref{eq:proof_not_in_intersection}.
\end{itemize}

\textbf{Case 2:} $Z_{\alpha(j)}$ is not a parent of some child of $Z_{\alpha(i)}$, denoted as $Z_{\alpha(l)}$. This means $l>i$. In this case, consider the latent DAG $\mathcal{G}_{Z_{\alpha[l]}}$, where $Z_{\alpha(l)}$ is a sink node. Clearly, $Z_{\alpha(j)}$ is not a parent, child, or spouse of $Z_{\alpha(l)}$ in DAG $\mathcal{G}_{Z_{\alpha[l]}}$. Therefore, $Z_{\alpha(j)}$ and $Z_{\alpha(l)}$ are not adjacent in the moral graph of $\mathcal{G}_{Z_{\alpha[l]}}$. By \cref{lemma:moral_graph_subgraph}, this indicates that $Z_{\alpha(j)}$ is not adjacent to $Z_{\alpha(l)}$ in the Markov network $\mathcal{M}_{Z_{\alpha[l]}}$, which, under the faithfulness assumption, is a neighbor of $Z_{\alpha(i)}$ in the Markov network $\mathcal{M}_{Z_{\alpha[l]}}$. This implies $Z_{\alpha(j)}\not\in\Psi(Z_{\alpha(i)};\mathcal{M}_{Z_{\alpha[l]}})$ and thus \cref{eq:proof_not_in_intersection}.
\end{proof}

\section{Proof of \cref{theorem:identifiability_hnm}}\label{app:theorem_identifiability_hnm}
\ThmIdentifiabilityHNM*
\begin{proof}
The proof here is identical to that of \cref{theorem:identifiability_anm}, which leverages \cref{proposition:induction_anm_dag,proposition:intimate_parents}. Therefore, we omit the proof here. The only (minor) difference is that the third-order derivative in \cref{eq:third_order_derivative} is instead given by:
\begin{flalign*}
0 = &\sum_{i:Z_{\alpha(i)}\in Z_{\alpha[n-t]}\setminus \mathcal{S}(\mathcal{G}_{Z_{\alpha[n-t]}})} \frac{\partial^3 \log p^{(u)}(Z_{\alpha[n-t]})}{\partial Z_{\alpha(i)}^3} \left(\frac{\partial  Z_{\alpha(i)}}{\partial \hat{Z}_{l}}\right)^3 \\
&\qquad+\sum_{\substack{i,j: \\\{Z_{\alpha(i)},Z_{\alpha(j)}\}\in \mathcal{E}(\mathcal{M}_{Z_{\alpha[n-t]}})}} \frac{\partial^3 \log p^{(u)}(Z_{\alpha[n-t]})}{\partial Z_{\alpha(i)}^2 \partial Z_{\alpha(j)}} \left(2\cdot\left(\frac{\partial  Z_{\alpha(i)}}{\partial \hat{Z}_{l}}\right)^2 \frac{\partial Z_{\alpha(j)}}{\partial \hat{Z}_{l}}\right)\\
&\qquad+\sum_{\substack{i,j,k:\\i<j<k,\\\{Z_{\alpha(i)},Z_{\alpha(j)}\},\{Z_{\alpha(j)},Z_{\alpha(k)}\},\{Z_{\alpha(i)},Z_{\alpha(k)}\}\in\mathcal{E}(\mathcal{M}_{Z_{\alpha[n-t]}})}} \frac{\partial^3 \log p^{(u)}(Z_{\alpha[n-t]})}{\partial Z_{\alpha(i)} \partial Z_{\alpha(j)}\partial Z_{\alpha(k)}} \left(6\cdot\frac{\partial  Z_{\alpha(i)}}{\partial \hat{Z}_{l}} \frac{\partial Z_{\alpha(j)}}{\partial \hat{Z}_{l}}\frac{\partial  Z_{\alpha(k)}}{\partial \hat{Z}_{l}}\right) \\
& \qquad+ \sum_{i=1}^d \frac{\partial^2 \log p^{(u)}(Z_{\alpha[n-t]})}{\partial Z_{\alpha(i)}^2} \left(3\cdot\frac{\partial^2  Z_{\alpha(i)}}{\partial \hat{Z}_{l}^2} \frac{\partial Z_{\alpha(i)}}{\partial \hat{Z}_{l}}\right)\\
& \qquad+ \sum_{\substack{i,j: \\i < j,\\ \{Z_{\alpha(i)},Z_{\alpha(j)}\}\in \mathcal{E}(\mathcal{M}_{Z_{\alpha[n-t]}})}} \frac{\partial^2 \log p^{(u)}(Z_{\alpha[n-t]})}{\partial Z_{\alpha(i)} \partial Z_{\alpha(j)}} \left(3\cdot\frac{\partial^2  Z_{\alpha(j)}}{\partial \hat{Z}_{l}^2} \frac{\partial Z_{\alpha(i)}}{\partial \hat{Z}_{l}}+3\cdot\frac{\partial  Z_{\alpha(j)}}{\partial \hat{Z}_{l}} \frac{\partial^2 Z_{\alpha(i)}}{\partial \hat{Z}_{l}^2}\right)\\
&\qquad + 
\sum_{i=1}^n \frac{\partial \log p^{(u)}(Z_{\alpha[n-t]})}{\partial  Z_{\alpha(i)}}\frac{\partial^3 Z_{\alpha(i)}}{\partial \hat{Z}_{l}^3} + \frac{\partial^3\log|\det J_{v}|}{\partial \hat{Z}_{l}^3}.
\end{flalign*}
\end{proof}

\section{Further Empirical Studies}
\subsection{Selecting Number of Latent Variables}\label{app:select_number_latent}
We provide empirical studies to demonstrate how to determine the number of latent variables. Specifically, according to \cref{fig:select_number_latent}, one can in practice perform model selection to select the number of latent variables based on the evidence lower bound (ELBO) loss.

\subsection{Selecting Hyperparameters}\label{app:select_sparsity}
We discuss how to select the hyperparameters for sparsity. Here, we choose $\lambda_1$ based on the ELBO loss and provide an example in \cref{fig:select_sparsity}. If the structure is overly sparse, it fails to reconstruct the input data and the reconstruction error will become high. If the structure is overly dense, it involves many unwanted edges and the KL divergence will become high.

\begin{figure*}[!h]
    \centering
   \begin{tabular}{ccc}
       \includegraphics[width=0.32\linewidth]{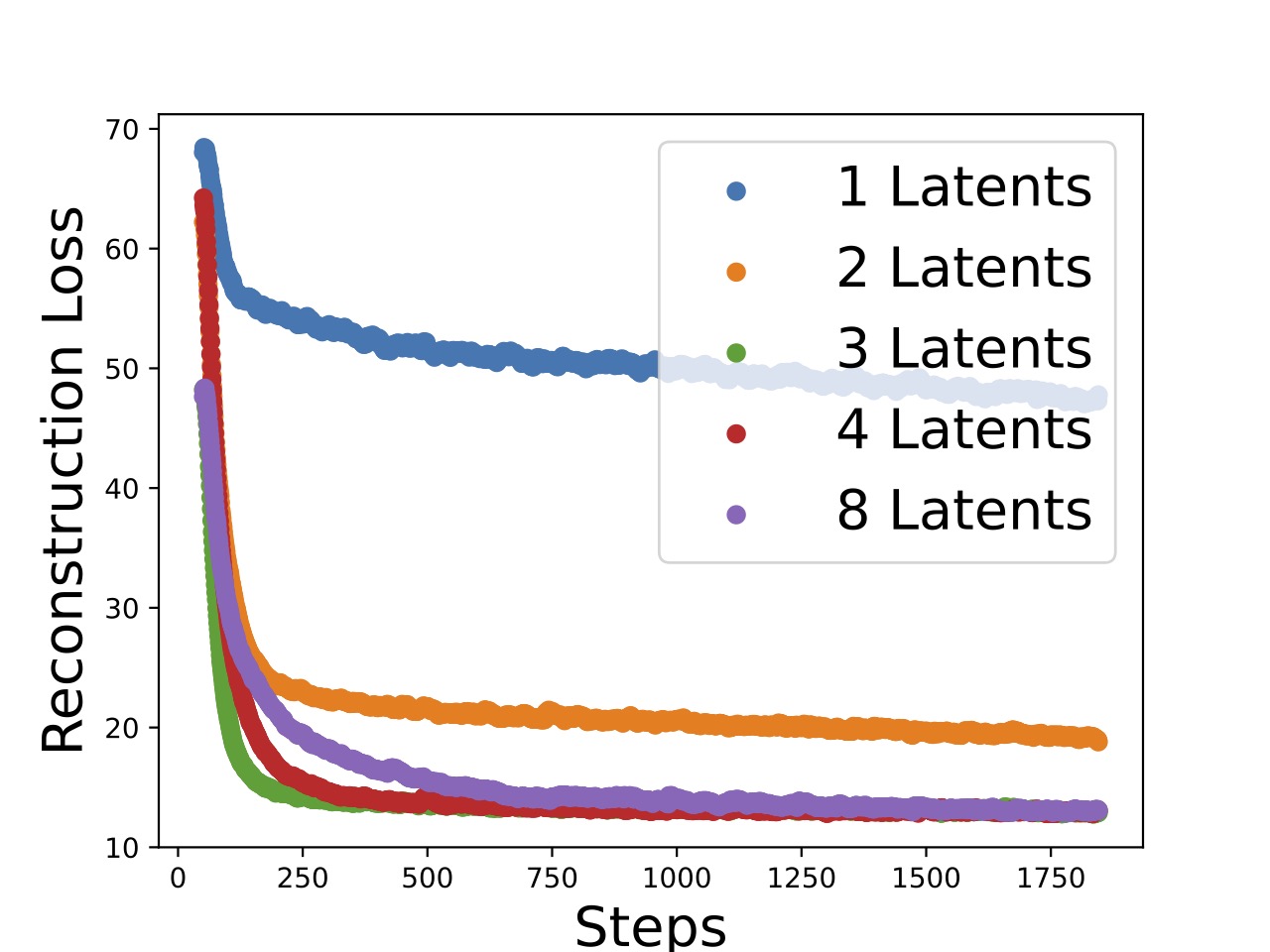} &
       \includegraphics[width=0.32\linewidth]{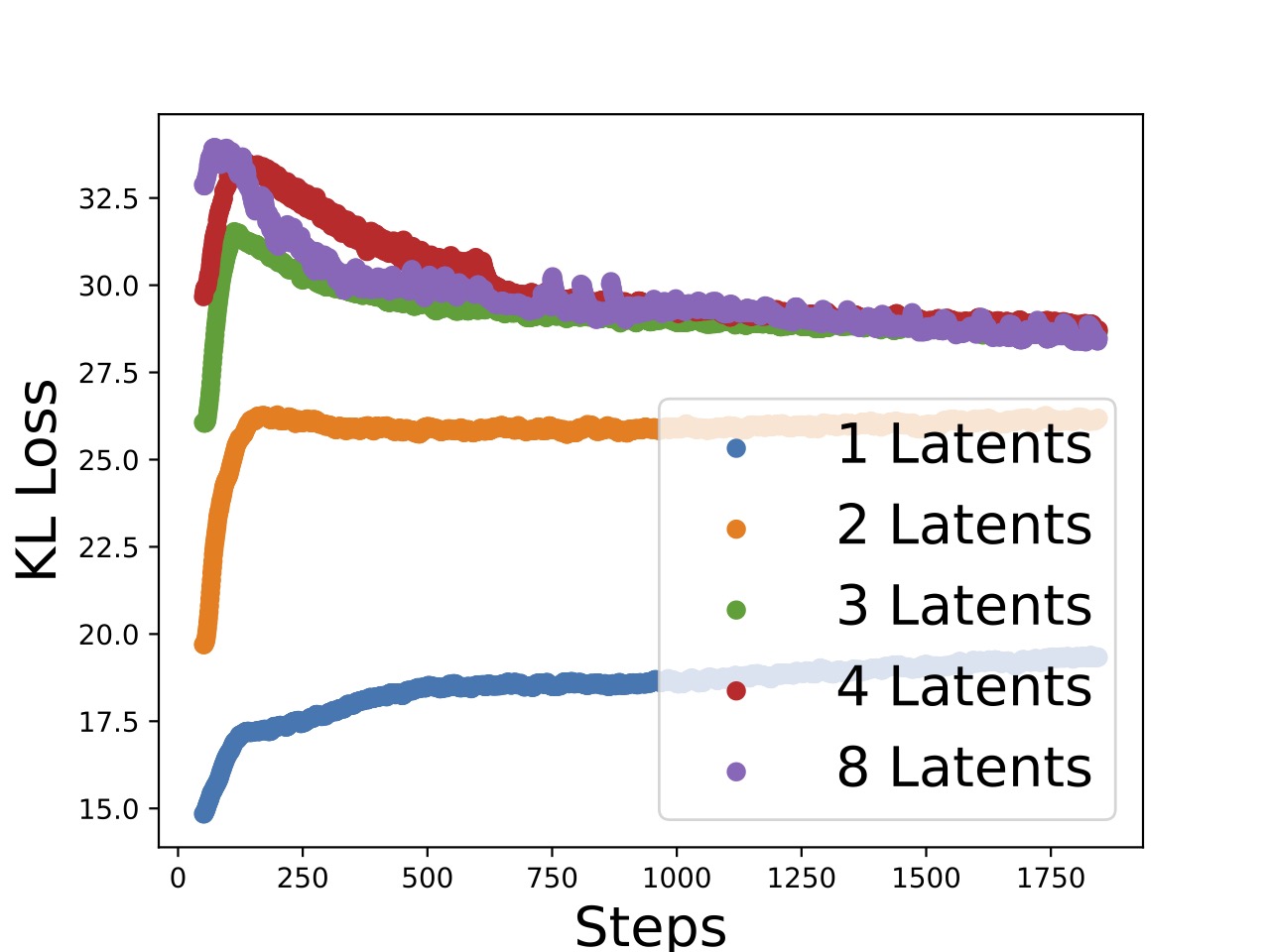} &
       \includegraphics[width=0.32\linewidth]{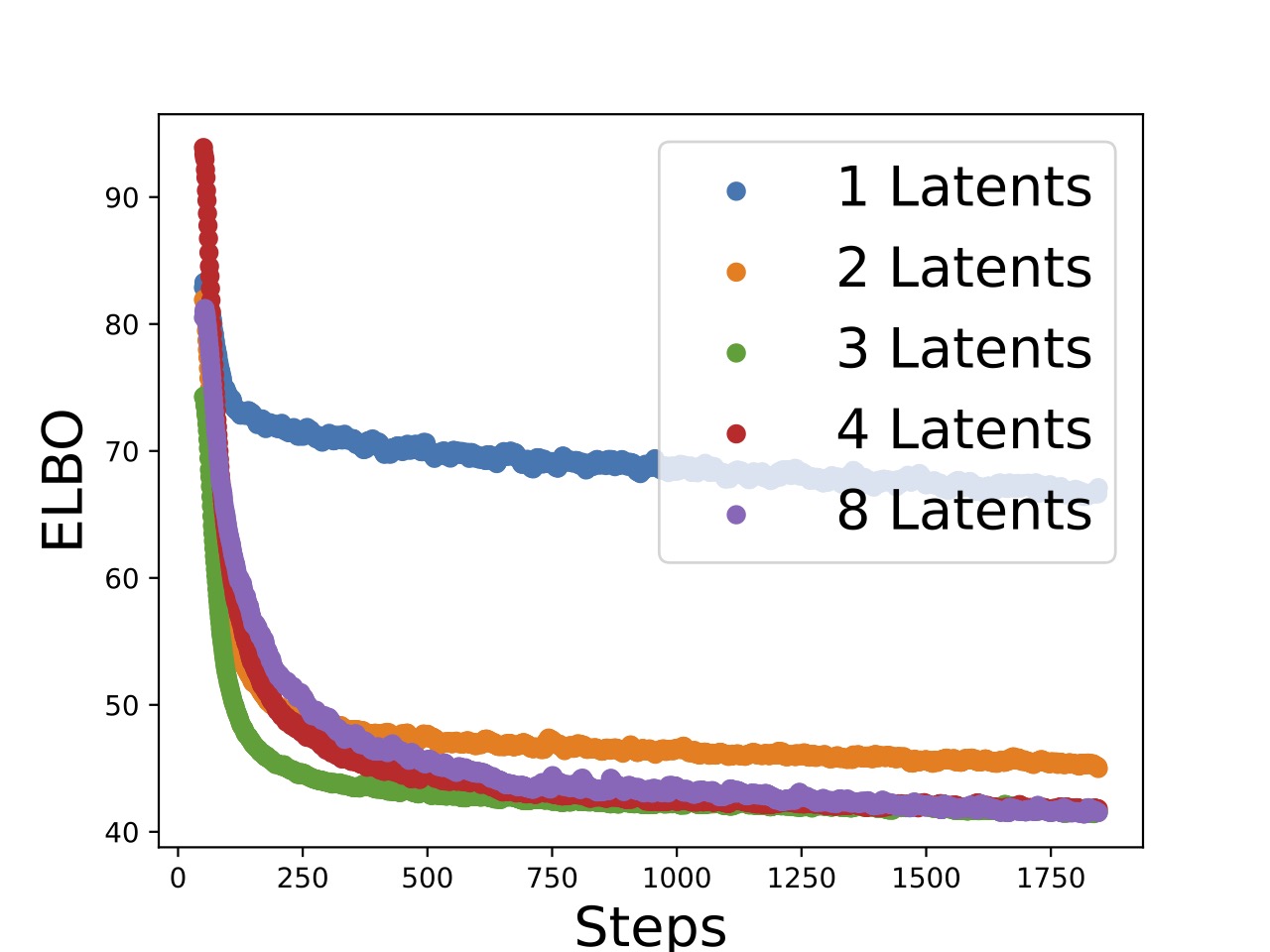}
   \end{tabular}
    \caption{The losses versus the training steps with different number of estimated latent variables. The ground truth number of latent variables is $3$. When we set the estimated number of latent variables to be low, e.g., $1$ or $2$, the reconstruction is poor, showing that they are unable to match the distributions. This suggests that the number of estimated latent variables can be selected based on the ELBO loss.}
    \label{fig:select_number_latent}
\end{figure*}

\begin{figure*}[!h]
    \centering
   \begin{tabular}{ccc}
       \includegraphics[width=0.32\linewidth]{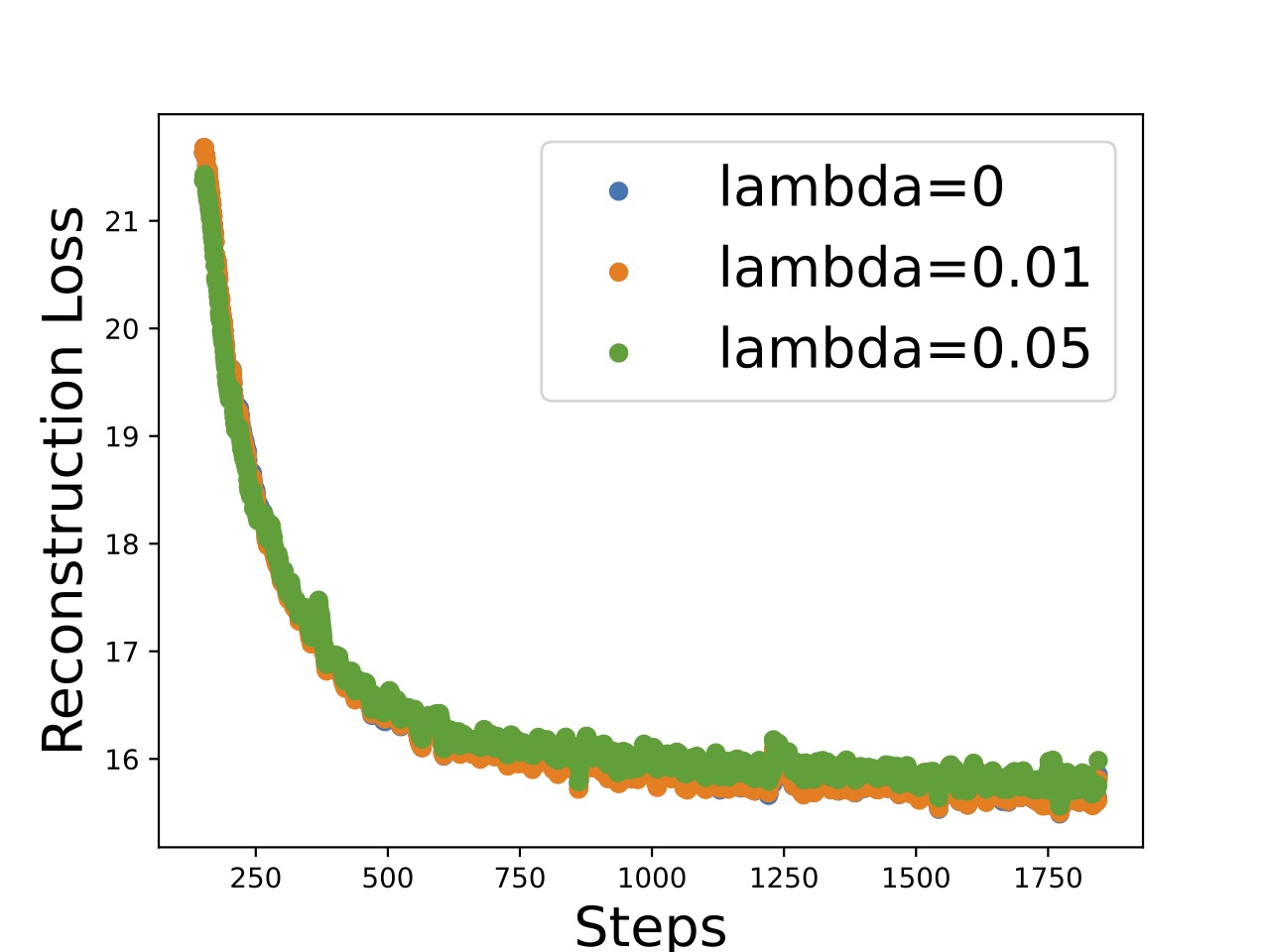} &
       \includegraphics[width=0.32\linewidth]{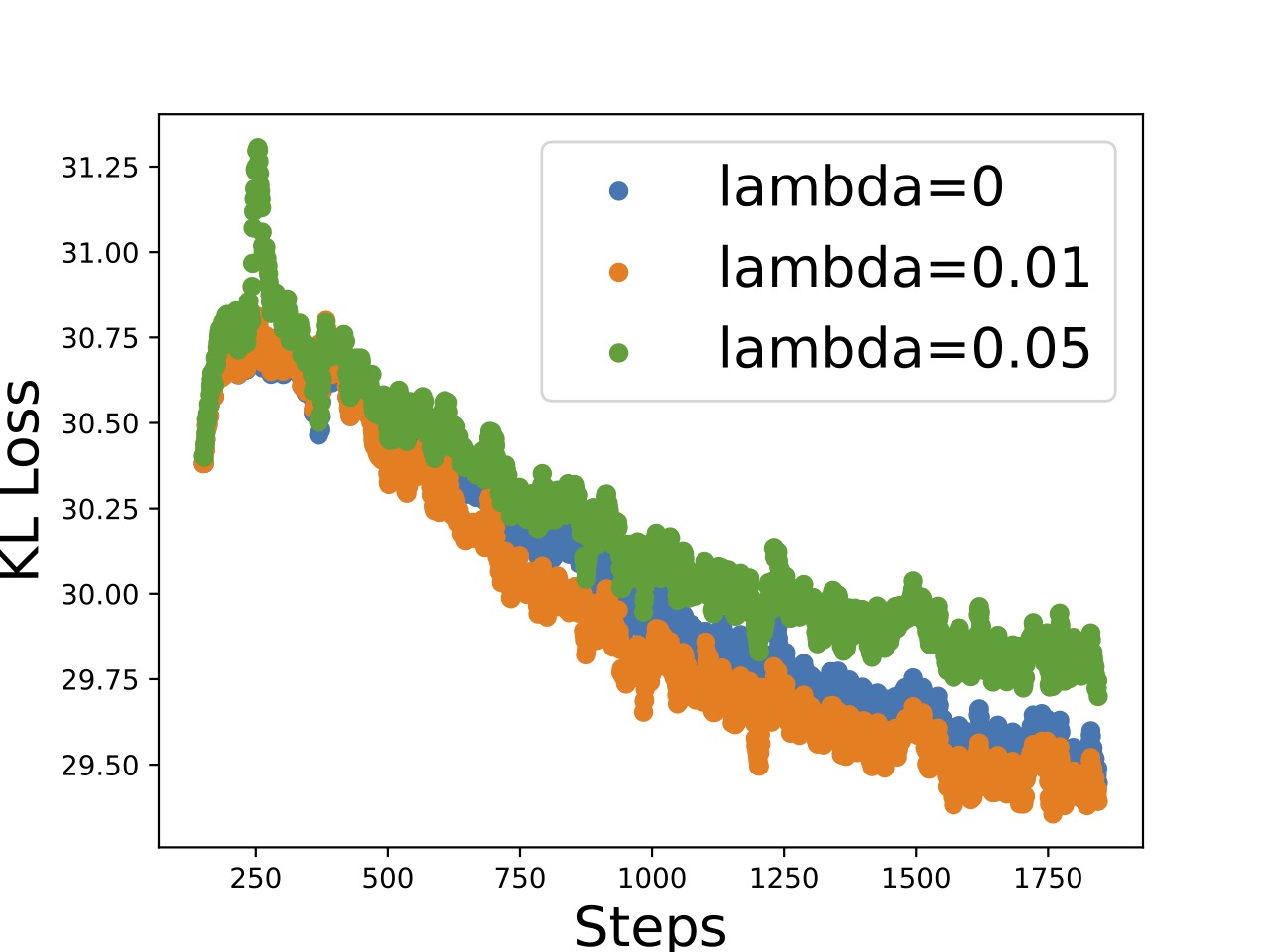} &
       \includegraphics[width=0.32\linewidth]{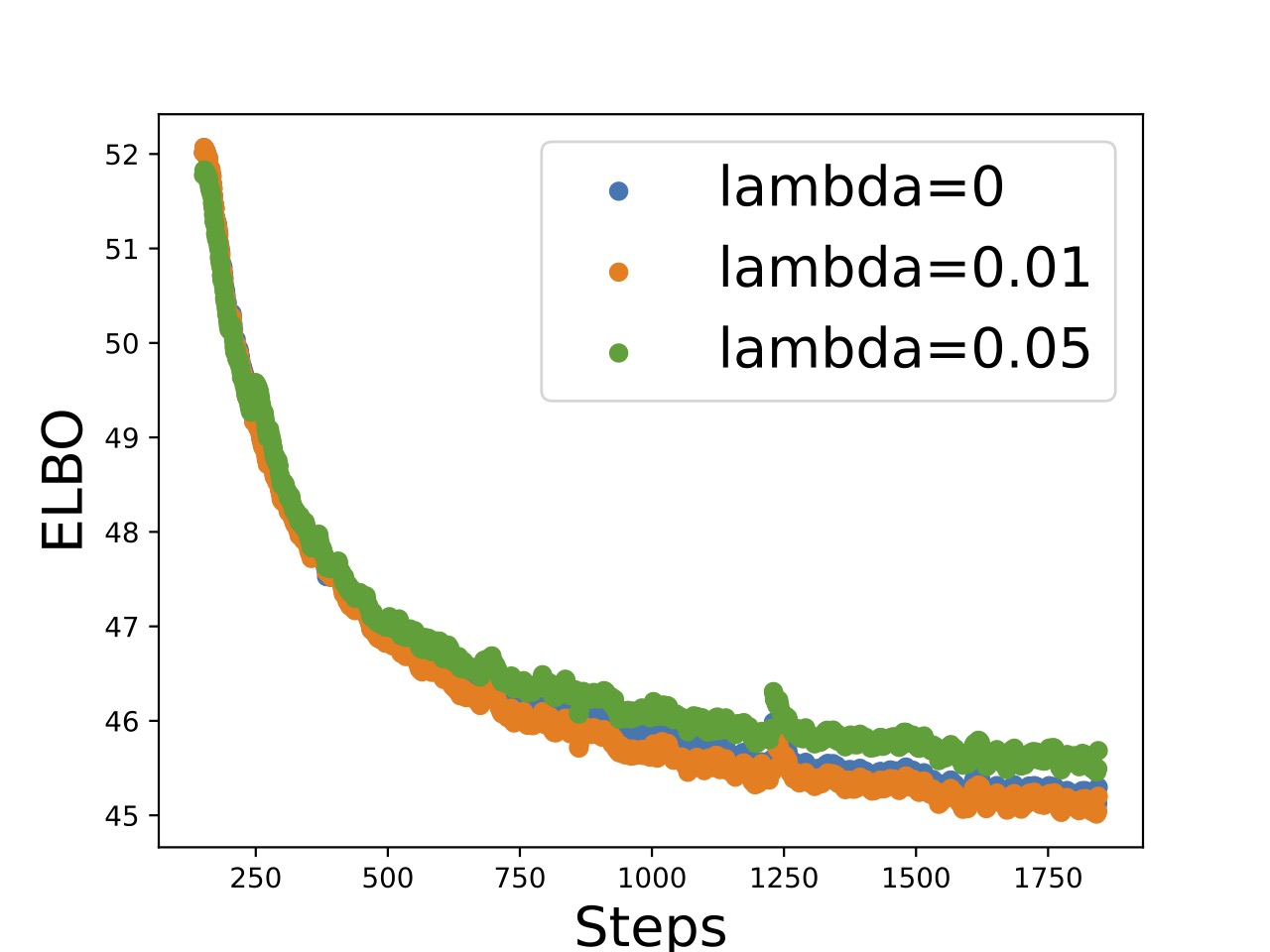}
   \end{tabular}
    \caption{The losses versus the training steps with different values of $\lambda_1$. If we set $\lambda_1=0.05$, which may be too high for the data, the ELBO loss is worst. If we set $\lambda_1=0$, the estimated graph is dense, i.e., estimated variables are unnecessarily dependent. When we set $\lambda_1=0.01$, we achieve an ELBO loss that is as good as the ELBO loss with $\lambda_1=0$, but with fewer edges.
    }
    \label{fig:select_sparsity}
\end{figure*}

\end{document}